\documentclass{article}

\PassOptionsToPackage{table}{xcolor}
\usepackage{report}

\usepackage[utf8]{inputenc} % allow utf-8 input
\usepackage[T1]{fontenc}    % use 8-bit T1 fonts
\usepackage{hyperref}       % hyperlinks
\usepackage{url}            % simple URL typesetting
\usepackage{booktabs}       % professional-quality tables
\usepackage{amsfonts}       % blackboard math symbols
\usepackage{nicefrac}       % compact symbols for 1/2, etc.
\usepackage{microtype}      % microtypography
\usepackage{xcolor}         % colors
\usepackage{graphicx}
\usepackage{amsmath}
\usepackage{subcaption}
\usepackage{amssymb}

\usepackage{amsmath}
\usepackage{amssymb}
\usepackage{mathtools}
\usepackage{amsthm}

\newtheorem{theorem}{Theorem}[section]          % 定理，按章节编号
\newtheorem{lemma}[theorem]{Lemma}              % 引理，与定理共享编号
\newtheorem{corollary}[theorem]{Corollary}      % 推论，与定理共享编号
\usepackage{algorithm}
\usepackage{algorithmic}
\usepackage{float}

\usepackage[most]{tcolorbox}
\tcbuselibrary{listings,breakable,skins}
\tcbset{
  promptbox/.style={
    enhanced,
    breakable,
    colback=gray!3,
    colframe=black!35,
    boxrule=0.6pt,
    arc=2pt,
    left=6pt,right=6pt,top=6pt,bottom=6pt,
    listing only,
    listing options={
      basicstyle=\ttfamily\footnotesize,
      breaklines=true,
      breakatwhitespace=false,
      columns=fullflexible,
      keepspaces=true,
      showstringspaces=false
    }
  }
}

\definecolor{ntuRed}{RGB}{199,16,45}
\definecolor{ntuBlue}{RGB}{0,61,126}
\definecolor{ntuGold}{RGB}{196,151,57}
\definecolor{ntuCream}{RGB}{255,250,241}

\newtcolorbox{abstractbox}{
    enhanced,
    breakable,
    colback=ntuCream,
    colframe=ntuRed,
    boxrule=0.9pt,
    arc=3mm,
    left=8pt,
    right=8pt,
    top=8pt,
    bottom=8pt
}

\usepackage{xcolor}
\usepackage{colortbl}
\usepackage{wrapfig}

\newcommand{\best}[1]{\cellcolor{ntuRed!24}{\strut \textbf{#1}}}
\newcommand{\second}[1]{\cellcolor{ntuGold!28}{\strut \textbf{#1}}}
\newcommand{\third}[1]{\cellcolor{ntuBlue!12}{\strut \textbf{#1}}}
\title{Text as Partial Constraint: Core--Residual Alignment for Robust Vision--Language Learning}

\author{%
Chengzhen Yu$^{1}$ \quad
Canran Xiao$^{2,*}$ \quad
Siyuan Ma$^{1}$ \quad
Yang Liu$^{1}$\\[4pt]
$^{1}$Nanyang Technological University (NTU)\\
$^{2}$Shenzhen Campus of Sun Yat-sen University\\
$^{*}$Corresponding author: \texttt{xiaocr3@mail.sysu.edu.cn}
}

\begin{document}

\maketitle

\begin{abstract}
\begin{abstractbox}
Vision--language alignment powers open-vocabulary recognition, retrieval, and LVLM grounding, yet natural captions are often underspecified, making similarity brittle and overly confident under paraphrase and omitted details. We aim to learn representations whose matching is stable across caption views and whose confidence reflects how strongly text constrains an image. We propose \textsc{Text as Partial Constraint} (\textsc{TPC}), a core--residual alignment framework that treats multi-view captions as incomplete supervision: it distills a consensus semantic core as the alignment target, learns a single-view core predictor for standard inference with one query, and explicitly discourages vision--language similarity from depending on the orthogonal ``unsaid'' residual. An uncertainty-aware contrastive objective further softens alignment when caption views disagree, reducing overconfident updates under weak language constraints. Across zero-shot recognition and adversarial robustness, \textsc{TPC} achieves 81.42/64.05 Top-1 clean/robust accuracy on ImageNet and 76.19/52.03 on an Avg-14 transfer suite, while improving LVLM transfer with 85.16 POPE F1 and 59.57 OKVQA accuracy under an LLaVA-1.5-7B stack. These results suggest that modeling text as a partial constraint is a practical and principled route to more reliable vision--language representations under underspecified language supervision.
\end{abstractbox}
\end{abstract}

\section{Introduction}
\label{sec:intro}
Vision--language alignment underpins open-vocabulary recognition, cross-modal retrieval, and increasingly serves as the perceptual backbone of large vision--language models (LVLMs) used for interactive reasoning and decision support \citep{pmlr-v139-radford21a,pmlr-v139-jia21b,liu2023llava,awadalla2023openflamingo}.
Its practical appeal comes from learning with abundant natural captions, yet the same supervision is inherently incomplete: users and annotators routinely omit details, rely on shared context, and phrase descriptions in diverse ways.
As these models move from curated benchmarks to real queries and safety-critical settings, failures driven by underspecified language become costly---mis-ranked retrieval, unstable predictions under paraphrase, and overconfident downstream generations.

Prior work has advanced vision--language pretraining by scaling contrastive dual encoders and improving training objectives \citep{pmlr-v139-radford21a,pmlr-v139-jia21b,zhai2023sigmoid,cherti2023reproducible,sun2023evaclip}.
Yet most methods still treat each caption/prompt as a \emph{complete} target, despite natural supervision being \emph{partial}: different textual views share a stable core but diverge in omitted or ambiguous details.
This mismatch surfaces in compositional and negation stress tests, where strong VLMs remain sensitive to small but meaning-relevant textual changes \citep{parcalabescu-etal-2022-valse,thrush_and_ross2022winoground,hsieh2023sugarcrepe,Alhamoud_2025_CVPR}, and in real under-specified user queries where adding missing constraints yields large gains \citep{choi2026underspecified}.
In parallel, robustness training mainly targets \emph{visual} shift \citep{mao2023understanding,Wang_2024_CVPR,pmlr-v235-schlarmann24a,dong2025robustsuperalignment}, while LVLM mitigation is often post hoc and does not curb representation-level over-commitment under weak text constraints \citep{li-etal-2023-evaluating,Leng_2024_CVPR}.
Thus, a key gap is the lack of a principled treatment of language supervision as \emph{incomplete} that prevents overconfident alignment to weakly supported, view-specific details.

This paper asks: \emph{Can we learn vision--language representations that align to what captions consistently specify, while avoiding over-confident commitment to what they leave underspecified?}
As illustrated in Fig.~\ref{fig:teaser}, we answer this question by reframing natural captions as partial constraints: multiple caption views reveal a shared consensus core, while their view-specific deviations expose an ``unsaid'' residual space.
Building on this observation, \textsc{TPC} aligns images to the consensus core, learns a single-caption core filter for standard inference, suppresses residual over-commitment, and calibrates confidence according to caption-view disagreement.

\begin{figure}[H]
	\centering
	\includegraphics[width=0.72\linewidth]{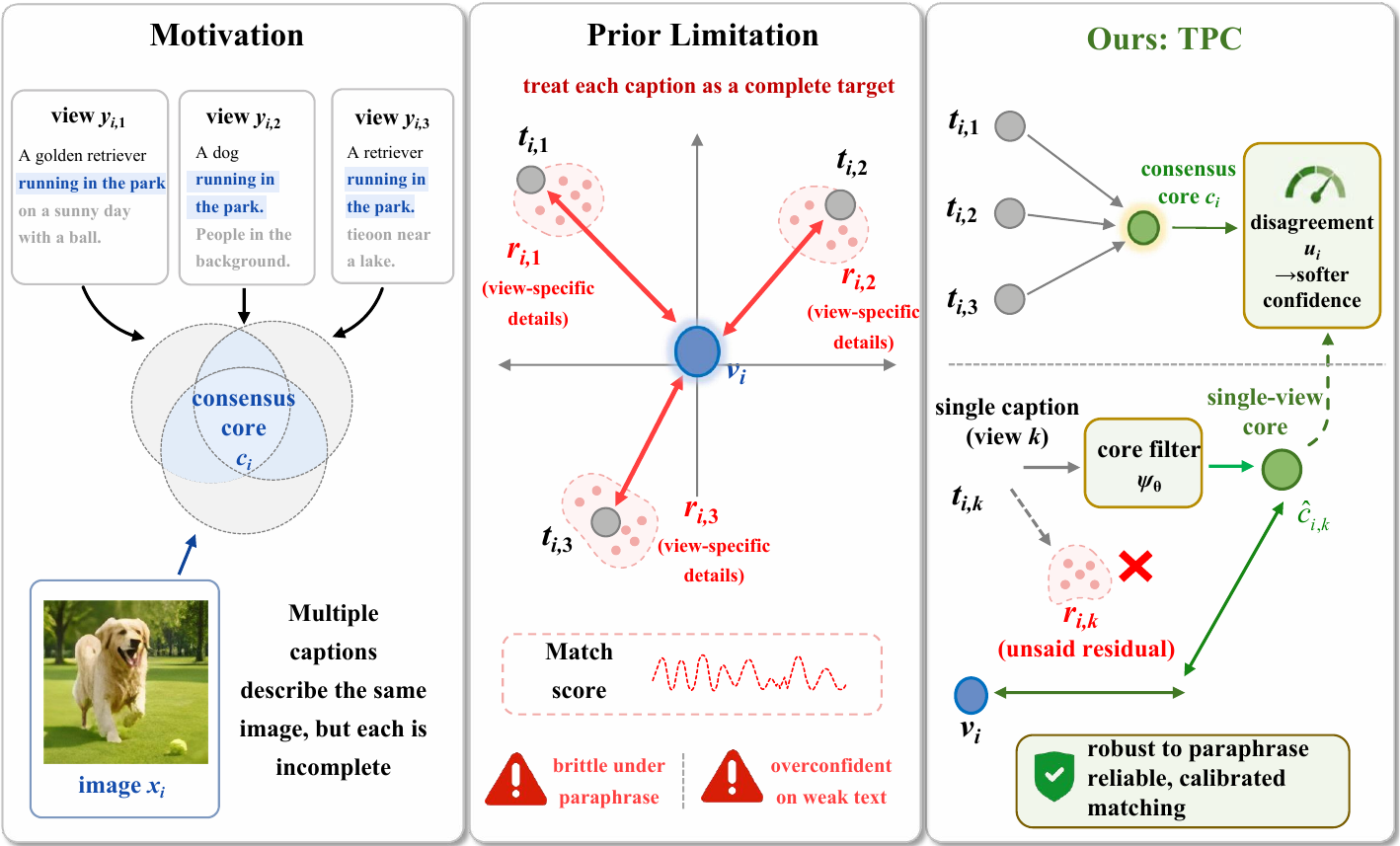}
	\caption{Unlike prior methods that over-align to each caption as a complete target, \textsc{TPC} extracts the core, suppresses residual commitment, and produces calibrated, robust vision--language alignment.}
	\label{fig:teaser}
\end{figure}

Our contributions are as follows: \textbf{\textit{(i)}} We formalize caption-view variation as a form of language underspecification and identify over-commitment to what is ``unsaid'' as a central mechanism behind brittleness and downstream hallucination.
\textbf{\textit{(ii)}} We introduce a training principle that uses multi-view language to learn alignment that prioritizes view-invariant semantics and produces confidence that tracks how strongly text constrains the instance, while remaining deployable with a single query text.
\textbf{\textit{(iii)}} Across zero-shot recognition, adversarial robustness, and LVLM transfer, the resulting representations improve both accuracy and reliability, yielding stronger robustness and reduced hallucination under underspecified language supervision.

\section{Related Work}
\label{sec:related}

\textbf{Contrastive VLP under noisy and underspecified text.}
CLIP and ALIGN showed that contrastive dual encoders trained on large-scale image--text pairs enable strong zero-shot transfer \citep{pmlr-v139-radford21a,pmlr-v139-jia21b}, and later work mostly improved results through scaling and refined objectives/recipes \citep{cherti2023reproducible,zhai2023sigmoid,sun2023evaclip}.
Yet natural captions are partial: annotators and paraphrases share core semantics but differ in omitted details, making similarity brittle under lexical/semantic edits and negation \citep{hsieh2023sugarcrepe,dumpala2024sugarcrepepp,Alhamoud_2025_CVPR} and under-specified real queries \citep{choi2026underspecified}.
Because training often treats each caption as complete (or merely adds positives), models can over-align to view-specific residuals.
\textsc{TPC} instead treats multi-caption supervision as a partial constraint, aligning to a consensus core while suppressing the ``unsaid'' residual so similarity reflects what is consistently supported.

\textbf{Robust alignment under distribution shift and attacks.}
Robust VLM training has focused on adversarial and distributional \emph{visual} shift.
Related methods include TeCoA \citep{mao2023understanding}, PMG-AFT \citep{Wang_2024_CVPR}, and TGA-ZSR \citep{yu2024tgazsr}, as well as plug-in robustness via unsupervised adversarial fine-tuning (RobustCLIP/FARE) \citep{pmlr-v235-schlarmann24a} and weak-to-strong robustness transfer (Adv-W2S) \citep{dong2025robustsuperalignment}.
While effective, they typically assume the text fully specifies the semantics and thus do not address caption-view ambiguity.
We instead treat caption variation as the shift source, model views via an instance-wise ambiguity set, and improve worst-case similarity by strengthening core alignment while reducing sensitivity to unsaid residuals, with disagreement-aware temperature scaling to curb overconfident updates.
Adjacent continual vision--language and multimodal adaptation work studies how to preserve compositional structure, prompt-invariant certificates, or concept-graph memory across task streams \citep{xiao2026reversible,zhang2026pi,zhou2026comem}.
These methods address temporal adaptation, while our focus is alignment under underspecified language supervision; the two directions are complementary.

\textbf{Grounding and hallucination in LVLMs.}
LVLMs such as LLaVA and OpenFlamingo largely inherit a frozen CLIP-like vision encoder with a lightweight connector \citep{liu2023llava,awadalla2023openflamingo}, and benchmarks like POPE highlight persistent object hallucination \citep{li-etal-2023-evaluating}.
Mitigations span decoding-time calibration (VCD) \citep{Leng_2024_CVPR}, connector-side fine-grained alignment \citep{jiang2025finegrained}, and post-hoc concept shaping (VL-SAE) \citep{shen2025vlsae}, among other LVLM alignment objectives \citep{truong2025directedtokens}.
Yet these fixes are constrained by the shared embedding space they build upon.
\textsc{TPC} improves this foundation by strengthening the alignment core and reducing residual over-commitment, serving as a drop-in vision encoder that complements decoding- and connector-level approaches.

\section{Preliminary Study}
\label{sec:preliminary_study}

Before introducing our method, we first diagnose a failure mode of frozen vision--language encoders under partial captions.
Using a frozen CLIP/OpenCLIP-style dual encoder, we evaluate images with multiple captions and ask:
\emph{when different captions describe the same image, does caption disagreement expose unstable and overconfident matching?}
No model is trained or modified in this study.
We compute four diagnostic quantities: caption-view dispersion, rank volatility across captions, residual leakage into view-specific textual components, and high-confidence retrieval error.
Full definitions and the evaluation protocol are provided in Appendix~\ref{app:prestudy_protocol}.

\begin{figure}[t]
\centering
	\includegraphics[width=0.98\linewidth]{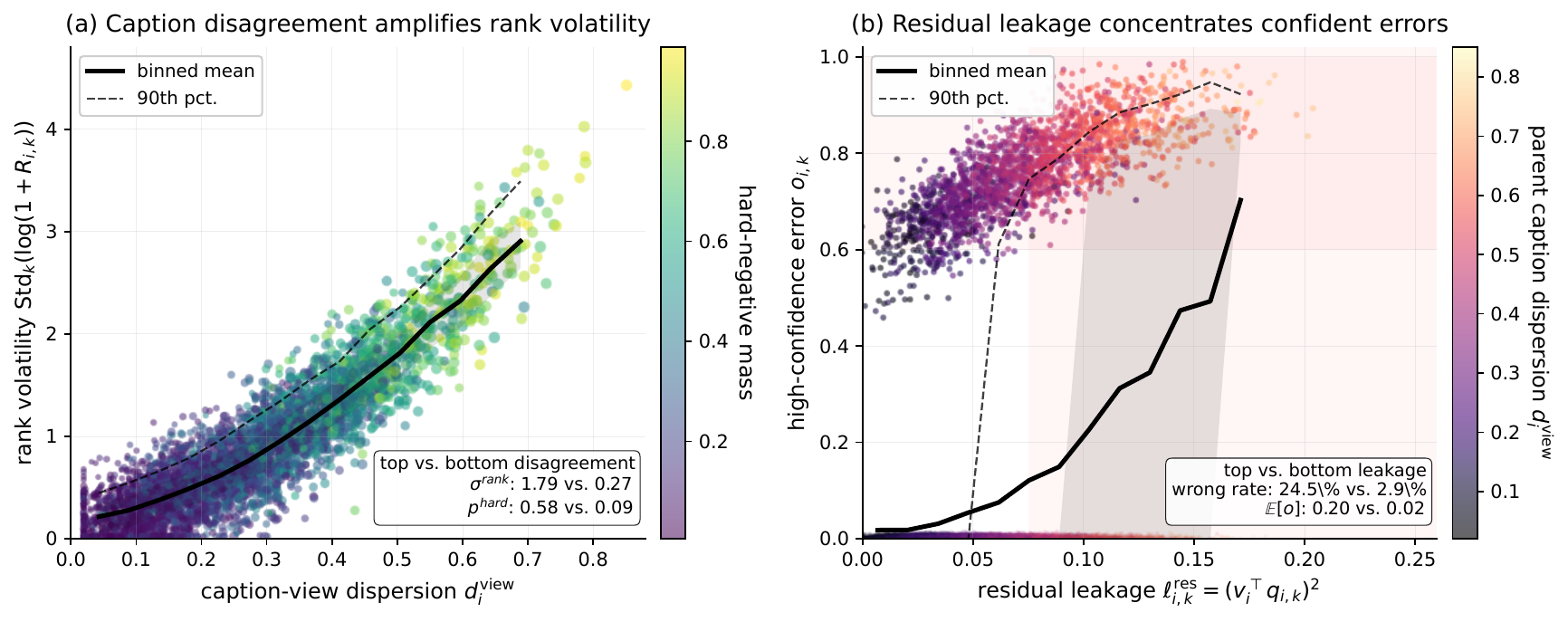}
\vspace{-0.6em}
\caption{\textbf{Frozen VLM failure under partial captions.}
Left: larger caption-view dispersion leads to higher rank volatility and stronger hard-negative confusion.
Right: residual leakage concentrates confident retrieval errors.
All statistics are computed with frozen embeddings; no proposed module or training loss is used.}
\label{fig:prestudy_failure_mechanism}
\vspace{-0.6em}
\end{figure}

\textbf{Finding 1: caption disagreement predicts unstable matching.}
Fig.~\ref{fig:prestudy_failure_mechanism}(a) shows that images with larger caption-view dispersion exhibit substantially higher rank volatility across captions.
The same image can be retrieved reliably with one valid caption but poorly with another, even though both captions refer to the same visual instance.
The hard-negative posterior mass also increases with disagreement, suggesting that underspecified captions make semantically nearby negatives increasingly competitive.

\textbf{Finding 2: view-specific residuals drive confident errors.}
Fig.~\ref{fig:prestudy_failure_mechanism}(b) shows that high-confidence retrieval errors concentrate in examples with large residual leakage.
This indicates that frozen VLM embeddings can over-commit to caption-specific components that are not consistently supported across views.
Such over-commitment is harmful not only because it changes the ranking, but also because the model often remains highly confident when it is wrong.

\textbf{Implication.}
These observations suggest that robust alignment under natural captions requires three ingredients:
(i) a stable semantic target shared across caption views,
(ii) a deployable mechanism that can recover this target from a single caption at inference time,
and (iii) an explicit way to prevent view-specific residuals from dominating similarity and confidence.
We now instantiate these requirements in a core--residual alignment framework.

\section{Method}
\label{sec:method}

\begin{figure}[htb]
	\centering
	\includegraphics[width=\linewidth]{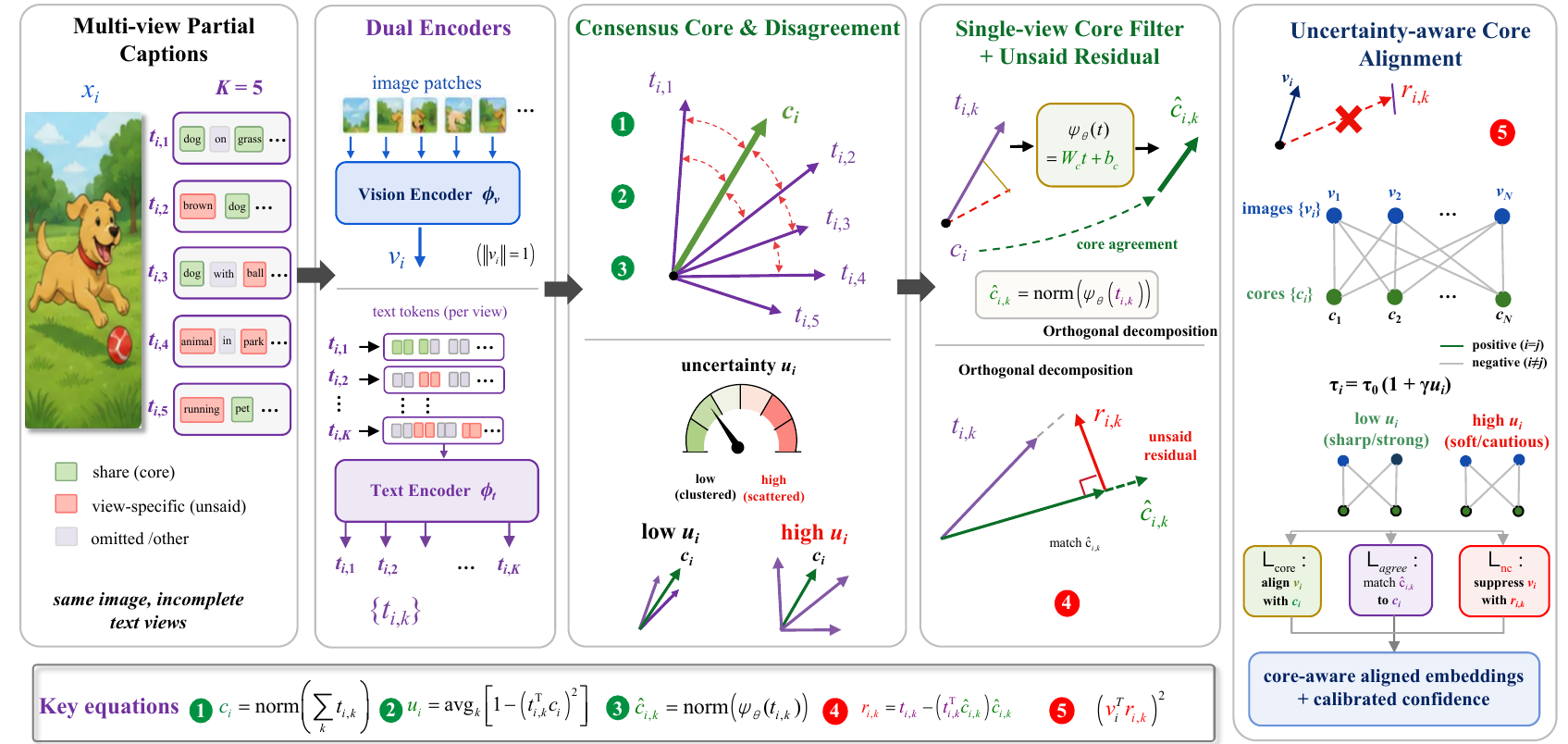}
	\caption{\textbf{Overview of \textsc{TPC}.}
\textsc{TPC} treats captions as partial constraints: it distills a multi-view consensus core, learns a single-view core predictor, decomposes each caption into core and unsaid residual components, and performs uncertainty-aware alignment that suppresses residual over-commitment.}
\label{fig:pipeline}
	\label{fig:pip}
\end{figure}

Motivated by the observation that multi-view captions share a stable core but differ in ``unsaid'' details that induce brittleness, we formulate \textsc{Text as Partial Constraint} as a minimal core--residual alignment framework, as shown in Fig.\ref{fig:pip}.
Given $K$ views, we align to a consensus core, learn a single-view core filter for one-text inference, and treat the orthogonal component as an unsaid residual that is discouraged from affecting similarity.
An uncertainty-aware contrastive objective with a non-commitment regularizer jointly optimizes these parts to improve robustness under underspecified language supervision.

\textbf{Setup.}
Different textual views of the same image agree on the main semantics but diverge on omitted or ambiguous details.
A robust aligner should match vision to the shared semantics and avoid committing to view-specific details that are not consistently supported.

In each minibatch, we sample $B$ instances $\{(x_i,\{y_{i,k}\}_{k=1}^K)\}_{i=1}^B$.
Let $\phi_v$ and $\phi_t$ be a vision encoder and a text encoder that map inputs to $\mathbb{R}^d$.
We use $\mathrm{norm}(z)\triangleq z/(\|z\|_2+\varepsilon)$ with $\varepsilon>0$.

\begin{equation}
v_i \triangleq \mathrm{norm}(\phi_v(x_i)) \in \mathbb{R}^d,
\qquad
t_{i,k} \triangleq \mathrm{norm}(\phi_t(y_{i,k})) \in \mathbb{R}^d .
\label{eq:embeddings}
\end{equation}
Here $v_i$ is the normalized image embedding, $t_{i,k}$ is the normalized embedding of the $k$-th text view for image $i$, and $d$ is the embedding dimension.

\subsection{Multi-View Consensus Core and Disagreement}
The intersection of semantics across views provides the most reliable supervision.
Disagreement across views signals underspecification and omissions, which should reduce confidence rather than sharpen alignment.

We define the consensus core as the normalized mean direction of multi-view text embeddings:
\begin{equation}
c_i \triangleq \mathrm{norm}\!\Big(\sum_{k=1}^{K} t_{i,k}\Big)\in\mathbb{R}^d .
\label{eq:consensus}
\end{equation}
The vector $c_i$ summarizes the shared semantics among $\{t_{i,k}\}$ and serves as the text-side target for alignment.

We quantify disagreement-based uncertainty by measuring the average squared cosine deviation from the consensus:
\begin{equation}
u_i \triangleq \frac{1}{K}\sum_{k=1}^{K}\Big(1-(t_{i,k}^{\top}c_i)^2\Big)\in[0,1].
\label{eq:uncertainty}
\end{equation}
Since $\|t_{i,k}\|_2=\|c_i\|_2=1$, each $(t_{i,k}^{\top}c_i)^2\in[0,1]$ and thus $u_i$ is nonnegative and bounded.
Large $u_i$ indicates that views scatter around $c_i$ and the language constraint is weak.

\subsection{Learning a Single-View Core Filter}
Benchmarks and applications provide a single query text at inference time.
We therefore learn a deterministic map that extracts the core from a single text embedding, using the multi-view consensus $c_i$ as supervision.

We introduce a lightweight core filter $\psi_\theta:\mathbb{R}^d\!\to\!\mathbb{R}^d$ implemented as a linear layer.
For each view embedding $t_{i,k}$, we predict a single-view core direction
\begin{equation}
\hat c_{i,k} \triangleq \mathrm{norm}\!\big(\psi_\theta(t_{i,k})\big)
= \mathrm{norm}\!\big(W_c t_{i,k}+b_c\big)\in\mathbb{R}^d,
\label{eq:core_filter}
\end{equation}
where $W_c\in\mathbb{R}^{d\times d}$ and $b_c\in\mathbb{R}^{d}$ are trainable parameters.
The agreement objective (defined below) enforces $\hat c_{i,k}$ to match the consensus $c_i$, making the mapping usable with a single text at test time.

\subsection{Unsaid Residual and Non-Commitment}
Even when the core is correct, the remaining components of language contain view-specific details and ambiguity.
To prevent over-alignment, we explicitly discourage the image embedding from aligning with these unsaid components.

Given the predicted core direction $\hat c_{i,k}$, we define the \emph{unsaid residual} as the component of $t_{i,k}$ orthogonal to $\hat c_{i,k}$:
\begin{equation}
r_{i,k} \triangleq t_{i,k} - (t_{i,k}^{\top}\hat c_{i,k})\,\hat c_{i,k}\in\mathbb{R}^d .
\label{eq:residual}
\end{equation}
Because $\|\hat c_{i,k}\|_2=1$, Eq.~\eqref{eq:residual} removes the projection of $t_{i,k}$ onto $\hat c_{i,k}$, yielding $\hat c_{i,k}^{\top}r_{i,k}=0$.

We enforce \emph{non-commitment} by penalizing correlation between image embeddings and residuals:
\begin{equation}
\mathcal{L}_{\mathrm{nc}}
\triangleq
\frac{1}{BK}\sum_{i=1}^{B}\sum_{k=1}^{K}
\big(v_i^{\top} r_{i,k}\big)^2 .
\label{eq:noncommit}
\end{equation}
Here $v_i^{\top}r_{i,k}$ is a scalar measuring image--residual alignment; squaring yields a smooth penalty that drives $v_i$ away from unsaid directions.

\subsection{Uncertainty-Aware Core Alignment}
When language is underspecified (large $u_i$), alignment should be less brittle and similarity should be less overconfident.
We implement this by scaling the contrastive sharpness using the disagreement-based uncertainty.

We set an instance-dependent temperature
$\tau_i \triangleq \tau_0(1+\gamma u_i)$ with $\tau_0>0$ and $\gamma\ge 0$.
We then align images to consensus cores using a symmetric InfoNCE loss:
\begin{equation}
\begin{aligned}
\mathcal{L}_{\mathrm{core}}
\triangleq & -\frac{1}{B}\sum_{i=1}^{B}
\log\frac{\exp\!\big(v_i^{\top}c_i/\tau_i\big)}{\sum_{j=1}^{B}\exp\!\big(v_i^{\top}c_j/\tau_i\big)} \\
& -\frac{1}{B}\sum_{i=1}^{B}
\log\frac{\exp\!\big(c_i^{\top}v_i/\tau_i\big)}{\sum_{j=1}^{B}\exp\!\big(c_i^{\top}v_j/\tau_i\big)} .
\label{eq:core_infonce}
\end{aligned}
\end{equation}
In Eq.~\eqref{eq:core_infonce}, negatives are other batch elements.
A larger $u_i$ increases $\tau_i$ and smooths the softmax distribution, reducing overconfident gradients when text constraints are weak.

\subsection{Training Objective}
We require the single-view core filter to reproduce the multi-view consensus and require vision to align to the consensus while ignoring residuals.

We train the core filter with a squared agreement loss
\begin{equation}
\mathcal{L}_{\mathrm{agree}}
\triangleq
\frac{1}{BK}\sum_{i=1}^{B}\sum_{k=1}^{K}
\|\hat c_{i,k}-c_i\|_2^2 .
\label{eq:agree}
\end{equation}
Eq.~\eqref{eq:agree} uses the consensus $c_i$ from Eq.~\eqref{eq:consensus} as a target, making $\hat c_{i,k}$ a single-view estimator of the shared semantics.

The full objective is
\begin{equation}
\mathcal{L}
\triangleq
\mathcal{L}_{\mathrm{core}}
+\lambda_{\mathrm{agree}}\mathcal{L}_{\mathrm{agree}}
+\lambda_{\mathrm{nc}}\mathcal{L}_{\mathrm{nc}},
\label{eq:total}
\end{equation}
with fixed weights $\lambda_{\mathrm{agree}}\ge 0$ and $\lambda_{\mathrm{nc}}\ge 0$.
We optimize parameters of $\phi_v$, $\phi_t$, and $\psi_\theta$ by minimizing Eq.~\eqref{eq:total}.

\subsection{Inference and Calibrated Confidence}
Standard retrieval uses a single text query and requires a single embedding.
We output a core embedding for ranking and a calibrated confidence for downstream decision-making.

Given a query text $y$, we compute $t=\mathrm{norm}(\phi_t(y))$ and $\hat c=\mathrm{norm}(\psi_\theta(t))$.
We rank candidates with the core similarity $s(x,y)=v^{\top}\hat c$.
We compute a scalar confidence temperature using the residual energy
$\hat u(y)\!=\!1-(t^{\top}\hat c)^2$ and $\tau(y)\!=\!\tau_0(1+\gamma \hat u(y))$,
then convert similarities into calibrated probabilities for downstream selection via
$p(j\mid y)=\mathrm{softmax}\!\big(s(x_j,y)/\tau(y)\big)$.
This calibration changes the sharpness of the retrieval distribution while preserving the ranking induced by $s(x,y)$.

\section{Theory}
\label{sec:theory}

We provide two theoretical justifications for \textsc{TPC}: caption-view variation induces a worst-case residual penalty, and multi-view consensus denoises incomplete textual supervision.

\subsection{Robustness to Caption-View Shift as Distributional Robust Optimization}
\label{sec:theory_dro}

Let $c_i$ be the consensus core of instance $i$ and let
$\mathcal{U}_i=\mathrm{span}\{t_{i,k}-(t_{i,k}^{\top}c_i)c_i:k\in[K]\}\subseteq c_i^\perp$
denote the subspace of caption-specific deviations.
With residual radius
$\rho_i=\max_k\|t_{i,k}-(t_{i,k}^{\top}c_i)c_i\|_2$,
we model admissible caption views by the spherical-cap ambiguity set
\begin{equation}
\mathcal{A}_i(\rho_i)\triangleq
\Bigl\{\,t\in\mathbb{R}^d:\ 
\begin{multlined}[t]
\|t\|_2=1,\ t=\sqrt{1-\|r\|_2^2}\,c_i+r,\\
r\in\mathcal{U}_i,\ \|r\|_2\le \rho_i\,\Bigr\}.
\end{multlined}
\label{eq:ambiguity_set}
\end{equation}
This set contains all observed caption views under the mild hemisphere condition and its radius is controlled by the disagreement statistic $u_i$; see Lemma~\ref{lem:contain_radius}.

\begin{theorem}[Tight invariance bound under caption-view shift]
\label{thm:robust_similarity}
Fix $c\in\mathbb{R}^d$ with $\|c\|_2=1$, a subspace $\mathcal{U}\subseteq c^\perp$, and $\rho\in[0,1)$.
Let $\mathcal{A}(c,\mathcal{U},\rho)$ be defined as in Eq.~\eqref{eq:ambiguity_set} by replacing $(c_i,\mathcal{U}_i,\rho_i)$ with $(c,\mathcal{U},\rho)$.
For any unit vector $v$, define
\begin{equation}
a\triangleq v^\top c,\qquad b\triangleq \|\mathbf{P}_{\mathcal{U}}v\|_2 .
\end{equation}
Then
\begin{equation}
\inf_{t\in\mathcal{A}(c,\mathcal{U},\rho)} v^\top t
=
a\sqrt{1-\rho^2}-\rho b .
\label{eq:worst_case_similarity}
\end{equation}
Consequently, for any monotone non-increasing loss $\ell$,
\begin{equation}
\sup_{t\in\mathcal{A}(c,\mathcal{U},\rho)}
\ell(v^\top t)
=
\ell\!\left(a\sqrt{1-\rho^2}-\rho b\right).
\label{eq:robust_loss_exact}
\end{equation}
\end{theorem}

Theorem~\ref{thm:robust_similarity} shows that robustness improves by increasing core alignment $a=v^\top c$ and decreasing residual sensitivity $b=\|\mathbf{P}_{\mathcal{U}}v\|_2$.
This directly motivates aligning image embeddings to consensus cores while suppressing their correlation with unsaid residual directions.
Since Lemma~\ref{lem:contain_radius} gives $\rho_i^2\le K u_i$, the same analysis also motivates using caption-view disagreement to modulate confidence.
The full proof is in Appendix~\ref{app:theory_dro}.

\subsection{Generalization via Multi-View Denoising}
\label{sec:theory_gen}

We further view multiple captions as noisy observations of a latent semantic core.
For each instance $i$, assume
\begin{equation}
t_{i,k}=\mu_i+\varepsilon_{i,k},\qquad
\mathbb{E}[\varepsilon_{i,k}\mid \mu_i]=0,\qquad
\|\varepsilon_{i,k}\|_2\le \sigma<1,
\label{eq:view_noise_model}
\end{equation}
where $\mu_i\in\mathbb{S}^{d-1}$ is the latent core.
Let $c_i=\mathrm{norm}(\frac{1}{K}\sum_k t_{i,k})$ be the normalized consensus.
For a predictor class $\mathcal{F}$ and an $L$-Lipschitz loss, define
\[
R_\mu(f)=\mathbb{E}\big[\ell(f(x)^\top\mu)\big],
\qquad
R_c(f)=\mathbb{E}\big[\ell(f(x)^\top c)\big].
\]

\begin{theorem}[PAC bound with $K$-view denoising]
\label{thm:pac_denoise}
Let $\widehat f\in\arg\min_{f\in\mathcal{F}}\widehat R_c(f)$, where
$\widehat R_c(f)=\frac{1}{n}\sum_{i=1}^{n}\ell(f(x_i)^\top c_i)$.
Then, with probability at least $1-\delta$,
\begin{equation}
\begin{split}
R_\mu(\widehat f)
&\le\inf_{f\in\mathcal{F}} R_\mu(f)
+ 4L\,\mathfrak{R}_n(\mathcal{F}) \\
&\quad+ 2\sqrt{\frac{\log(4/\delta)}{2n}} 
+ 2L\,\varepsilon_{K,n}(\delta),
\end{split}
\label{eq:pac_denoise_bound}
\end{equation}
where
\begin{align}
\varepsilon_{K,n}(\delta) &\triangleq \frac{2\eta_{K,n}(\delta)}{1-\eta_{K,n}(\delta)}, &
\eta_{K,n}(\delta) &\triangleq \frac{\sigma}{\sqrt{K}} + \sigma\sqrt{\frac{2\log(2n/\delta)}{K}}.
\label{eq:eps_kn}
\end{align}
In particular, when $\eta_{K,n}(\delta)\le \tfrac{1}{2}$,
$\varepsilon_{K,n}(\delta)=O(\sigma\sqrt{\log(n/\delta)/K})$.
\end{theorem}

\begin{corollary}[Sample complexity in $K$]
\label{cor:k_sample_complexity}
Under the conditions of Theorem~\ref{thm:pac_denoise}, to make the denoising contribution at most $\epsilon$, it suffices that
\begin{equation}
K=\Omega\!\left(\frac{\sigma^2\log(n/\delta)}{\epsilon^2}\right).
\end{equation}
\end{corollary}

Theorem~\ref{thm:pac_denoise} and Corollary~\ref{cor:k_sample_complexity} show that the consensus target becomes statistically cleaner as the number of views increases.
This supports using multi-view consensus for training while learning a single-view core estimator for deployment.
The full proof is in Appendix~\ref{app:theory_gen}.
\section{Experiments}
\label{sec:exp}

% =========================
% Main paper (concise)
% =========================
\subsection{Experimental Setup}
\label{sec:exp_setup}

\textbf{Datasets.}
We evaluate \textsc{Text as Partial Constraint} on three complementary suites.
(i) \emph{Image--text retrieval}: MS-COCO and Flickr30K, which provide multiple captions per image and are standard for alignment evaluation.
(ii) \emph{Zero-shot recognition}: ImageNet-1K and a broad transfer suite covering natural, fine-grained, texture, remote sensing, medical, and sketch/rendition shifts.
(iii) \emph{Robustness to underspecified language}: datasets with multi-caption annotations (COCO/Flickr30K) where we can evaluate sensitivity to caption choice.
We additionally report LVLM-side results by swapping in our vision encoder for fixed LVLM frameworks.

\textbf{Evaluation metrics.}
For retrieval, we report Recall@\{1,5,10\} for both image-to-text and text-to-image.
For zero-shot recognition, we report Top-1 accuracy and, when applicable, robust accuracy under standardized adversarial evaluation.
To assess overconfidence under weak language constraints, we report calibration metrics (ECE and NLL) on both classification probabilities and retrieval softmax scores (Appendix~\S\ref{app:exp_metrics}).

\textbf{Compared methods.}
We compare \textsc{Text as Partial Constraint} to:
(i) \emph{contrastive dual-encoder baselines} and strong recent training recipes (CLIP~\citep{pmlr-v139-radford21a}, OpenCLIP~\citep{cherti2023reproducible}, SigLIP~\citep{zhai2023sigmoid}, EVA-CLIP~\citep{sun2023evaclip});
(ii) \emph{robust/anti-misalignment alignment} methods targeting zero-shot robustness or robustness transfer (TeCoA~\citep{mao2023understanding}, RobustCLIP/FARE~\citep{pmlr-v235-schlarmann24a}, Robust SuperAlignment (Adv-W2S)~\citep{dong2025robustsuperalignment});
(iii) \emph{post-hoc representation enhancement} at the concept level (VL-SAE~\citep{shen2025vlsae});
and (iv) \textbf{LVLM-side alignment} under a fixed LVLM stack (LLaVA~\citep{liu2023llava}, OpenFlamingo~\citep{awadalla2023openflamingo}), including patch-aligned training~\citep{jiang2025finegrained} and Directed-Tokens~\citep{truong2025directedtokens}.
We match backbone/data/steps when feasible; methods requiring adversarial generation or extra supervision are additionally reported in a separate, compute-matched block.
Refer to ~\S\ref{app:impl_details} for Implementation details.

\subsection{Main Results}
\label{sec:main_results}

\begin{table}[t]
\centering
\resizebox{\linewidth}{!}{%
\begin{tabular}{lcccc}
\toprule
Method & ImageNet (Clean$\uparrow$) & ImageNet (Robust$\uparrow$) & Avg-14 (Clean$\uparrow$) & Avg-14 (Robust$\uparrow$) \\
\midrule
Standard CLIP~\citep{pmlr-v139-radford21a}
& 74.90 & 0.00 & \second{73.08} & 0.01 \\
TeCoA~\citep{mao2023understanding}
& \third{80.00} & \second{61.74} & 61.56 & 43.26 \\
PMG-AFT~\citep{wang2024pmgaft}
& 77.84 & 60.02 & 64.46 & \third{45.74} \\
FARE~\citep{pmlr-v235-schlarmann24a}
& 72.96 & 43.56 & 65.50 & 42.97 \\
TGA-ZSR~\citep{yu2024tgazsr}
& \best{80.26} & \third{61.46} & 62.11 & 45.19 \\
Adv-W2S~\citep{dong2025robustsuperalignment}
& 75.84 & 58.30 & \third{68.75} & \second{48.38} \\
\midrule
\textsc{TPC} (Ours)
& \best{81.42} & \best{64.05} & \best{76.19} & \best{52.03} \\
\bottomrule
\end{tabular}%
}
\caption{\textbf{Zero-shot recognition and adversarial robustness.}
Top-1 clean / robust accuracy (\%) on ImageNet and the average over 14 datasets under AutoAttack ($\epsilon{=}2/255$).}
\label{tab:main_cvlm}
\end{table}

\begin{table}[t]
\centering
\setlength{\tabcolsep}{4.6pt}
\small
\caption{\textbf{Transfer to LVLMs: hallucination and VQA.}
Left: POPE hallucination F1 (\%) under three sampling schemes and their average (higher is better).
Right: VQA accuracy (\%) under an LLaVA-1.5-7B stack.}
\label{tab:main_lvlm}
\resizebox{\linewidth}{!}{%
\begin{tabular}{lcccc|cccc}
\toprule
& \multicolumn{4}{c}{\textbf{POPE Hallucination (F1$\uparrow$)}} & \multicolumn{4}{c}{\textbf{VQA Accuracy$\uparrow$}} \\
\cmidrule(lr){2-5}\cmidrule(lr){6-9}
Method & Random & Popular & Adversarial & Avg & GQA & SciQA & VizWiz & OKVQA \\
\midrule
\multicolumn{9}{l}{\textit{Decoding / concept-level enhancement (frozen LVLM)}}\\
LLaVA1.5 (Regular)~\citep{liu2023llava}
& 80.87 & 79.27 & 77.16 & 79.10 & \third{62.0} & \third{66.8} & \third{50.0} & \third{53.4} \\
VCD~\citep{leng2024vcd}
& 84.04 & 82.31 & 80.13 & 82.16 & 60.8 & 65.2 & 48.7 & 51.9 \\
VL-SAE~\citep{shen2025vlsae}
& \third{85.50} & \third{84.37} & \second{82.29} & \third{84.05} & 61.5 & 66.1 & 49.8 & 53.0 \\
\midrule
\multicolumn{9}{l}{\textit{Projector / fine-grained alignment (LVLM-side training)}}\\
Patch-Aligned Training~\citep{jiang2025finegrained}
& 81.2 & 80.5 & 78.1 & 79.93 & \second{63.0} & \second{68.7} & \second{52.3} & \second{58.3} \\
\midrule
\multicolumn{9}{l}{\textit{Robust vision encoder replacement}}\\
TeCoA~\citep{mao2023understanding}
& 79.80 & 79.10 & 75.20 & 78.00 & 60.2 & 64.5 & 47.3 & 51.8 \\
PMG-AFT~\citep{wang2024pmgaft}
& 81.70 & 80.90 & 76.30 & 79.60 & 61.0 & 65.8 & 48.2 & 52.6 \\
FARE~\citep{pmlr-v235-schlarmann24a}
& 82.20 & 81.50 & 78.60 & 80.80 & 61.8 & 66.5 & 49.5 & 54.1 \\
TGA-ZSR~\citep{yu2024tgazsr}
& 80.40 & 79.80 & 76.00 & 78.70 & 60.5 & 64.9 & 47.8 & 52.0 \\
Adv-W2S~\citep{dong2025robustsuperalignment}
& \second{85.60} & \second{84.90} & \third{81.00} & \second{83.80} & 62.5 & 67.9 & 51.4 & 56.8 \\
\midrule
\multicolumn{9}{l}{\textit{Ours (core--residual alignment)}}\\
\textsc{TPC} (Ours)
& \best{86.93} & \best{85.81} & \best{82.74} & \best{85.16}
& \best{63.38} & \best{69.12} & \best{52.94} & \best{59.57} \\
\bottomrule
\end{tabular}
} % resizebox
\end{table}

%Tables~\ref{tab:main_cvlm}--\ref{tab:main_lvlm} show that modeling text as a \emph{partial constraint} yields consistent gains.

\textbf{Zero-shot generalization \& robustness.}
In Table~\ref{tab:main_cvlm}, \textsc{TPC} is best on both clean and robust accuracy, reaching 81.42/64.05 on ImageNet and 76.19/52.03 on Avg-14 (clean/robust), outperforming the strongest robust baseline Adv-W2S by +5.58/+5.75 on ImageNet and +7.44/+3.65 on Avg-14.
This supports that isolating a shared core and suppressing residual alignment reduces brittle over-commitment.

\textbf{LVLM transfer.}
In Table~\ref{tab:main_lvlm}, using \textsc{TPC} as a drop-in vision encoder achieves the best POPE F1 across all settings (Avg 85.16), exceeding Adv-W2S by +1.36.
It also improves VQA accuracy on all benchmarks (e.g., OKVQA 59.57), surpassing Patch-Aligned Training, indicating that core--residual alignment complements projector- or decoding-level fixes.

\subsection{Ablation and Analysis}
\label{sec:ablation}

\newcommand{\dec}[1]{\textcolor{blue!70}{\scriptsize\,($\downarrow$#1)}}

\textbf{Single-factor ablations.}
Table~\ref{tab:ablation} shows that each component matters.
Removing the multi-view consensus core most harms robustness and LVLM transfer (Avg-14 robust \dec{2.30}, POPE \dec{1.74}), indicating the shared intersection is a better target than any single caption.
Uncertainty scaling mainly affects robust accuracy (ImageNet robust \dec{1.17}), while dropping core agreement yields the largest overall regression (ImageNet robust \dec{3.13}), underscoring the need for a reliable single-view core at test time.
Removing non-commitment also consistently degrades robustness and hallucination metrics (Avg-14 robust \dec{2.37}, POPE \dec{2.13}), supporting residual suppression.

\begin{table}[t]
\centering
\setlength{\tabcolsep}{6.0pt}
\small
\caption{\textbf{Single-factor ablations of \textsc{TPC}.}
We report zero-shot clean/robust accuracy (\%) and LVLM transfer (POPE Avg F1, VQA Avg).}
\label{tab:ablation}
\resizebox{\linewidth}{!}{%
\begin{tabular}{lcccccc}
\toprule
Method
& ImageNet (Clean$\uparrow$)
& ImageNet (Robust$\uparrow$)
& Avg-14 (Clean$\uparrow$)
& Avg-14 (Robust$\uparrow$)
& POPE (Avg F1$\uparrow$)
& VQA (Avg$\uparrow$) \\
\midrule
\textsc{TPC} (Full)
& \textbf{81.42} & \textbf{64.05} & \textbf{76.19} & \textbf{52.03} & \textbf{85.16} & \textbf{61.25} \\
\midrule
\quad w/o multi-view consensus core ($c_i{\leftarrow}t_{i,k}$)
& 80.61\dec{0.81} & 62.47\dec{1.58} & 75.18\dec{1.01} & 49.73\dec{2.30} & 83.42\dec{1.74} & 60.21\dec{1.04} \\
\quad w/o uncertainty scaling ($\gamma{=}0$)
& 81.07\dec{0.35} & 62.88\dec{1.17} & 75.73\dec{0.46} & 50.21\dec{1.82} & 84.11\dec{1.05} & 60.82\dec{0.43} \\
\quad w/o core agreement ($\lambda_{\mathrm{agree}}{=}0$)
& 80.03\dec{1.39} & 60.92\dec{3.13} & 74.62\dec{1.57} & 48.31\dec{3.72} & 82.96\dec{2.20} & 59.74\dec{1.51} \\
\quad w/o non-commitment ($\lambda_{\mathrm{nc}}{=}0$)
& 81.16\dec{0.26} & 62.34\dec{1.71} & 75.50\dec{0.69} & 49.66\dec{2.37} & 83.03\dec{2.13} & 60.05\dec{1.20} \\
\bottomrule
\end{tabular}%
}
\end{table}

\textbf{Hyperparameter sensitivity.}
Fig.~\ref{tab:sens} shows \textsc{TPC} is stable across a wide range of settings.
Robustness increases with the number of views and saturates around $K{\approx}4$--$5$; $\gamma$ exhibits a mild sweet spot (too small under-calibrates, too large over-smooths).
$\lambda_{\mathrm{nc}}$ primarily controls robustness/hallucination and remains stable for $\lambda_{\mathrm{nc}}\in[0.05,0.2]$, while $\lambda_{\mathrm{agree}}$ is flat around 1--2.
Varying $\tau_0$ within the standard CLIP range changes little, suggesting calibration is dominated by the relative modulation from $u$.

\begin{figure}[t]
	\centering
		\includegraphics[width=\linewidth]{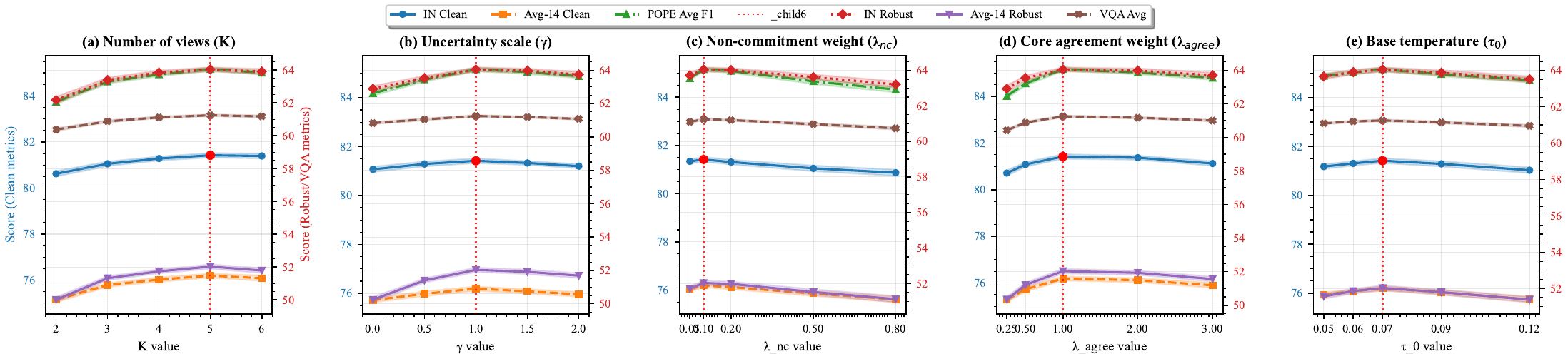}
	\caption{\textbf{Hyperparameter sensitivity of \textsc{TPC}.}
Mean$\pm$std over 3 seeds.}
\label{tab:sens}
\end{figure}

\paragraph{Core--Residual Mechanistic Evidence.}
We test whether the non-commitment regularizer reduces alignment to view-specific residual directions.
On multi-caption data, we measure residual leakage $\ell=(v^\top r)^2$ before and after fine-tuning for standard CLIP and \textsc{TPC}, using the caption consensus to define the residual space.
Fig.~\ref{fig:residual_leakage} shows that standard CLIP fine-tuning increases leakage from $\approx0.049$ to $\approx0.071$, whereas \textsc{TPC} reduces it to $\approx0.040$ and tightens the tail.
This supports that $\lambda_{\mathrm{nc}}$ suppresses over-commitment to unsaid, view-specific components rather than merely improving aggregate accuracy.

\begin{figure}[H]
\centering
\includegraphics[width=0.62\linewidth]{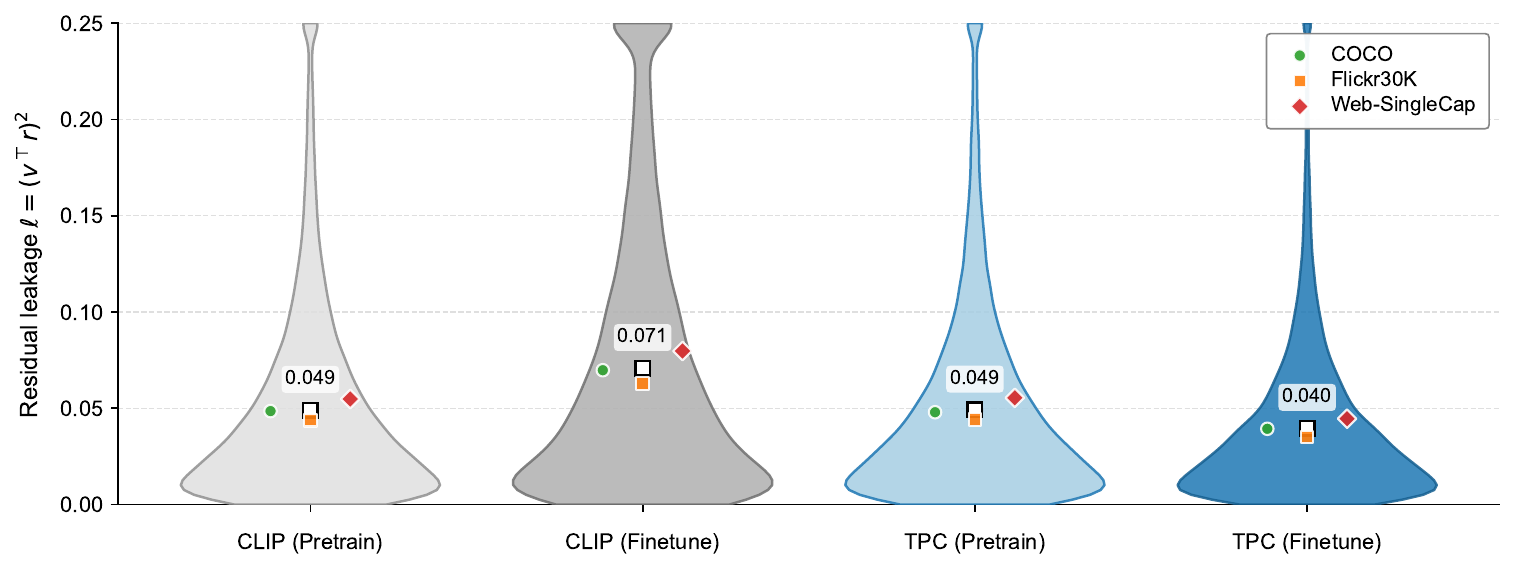}
\caption{\textbf{Residual leakage.}
Distribution of $\ell=(v^\top r)^2$ before/after fine-tuning; lower indicates less residual commitment.
Protocol in App.~\ref{app:residual_leakage_protocol}.}
\label{fig:residual_leakage}
\end{figure}

\paragraph{Uncertainty Calibration.}
We evaluate whether the inferred uncertainty $\hat u$ ranks queries by reliability under weak text constraints.
We sort queries from low to high $\hat u$ and compute risk--coverage curves by retaining increasingly uncertain queries.
Fig.~\ref{fig:risk_coverage_u} shows that \textsc{TPC} achieves lower risk than CLIP at the same coverage, with the largest gains in the selective low-coverage regime.
This indicates that $\hat u$ is a meaningful abstention signal for underspecified text.

\begin{figure}[H]
\centering
\includegraphics[width=0.62\linewidth]{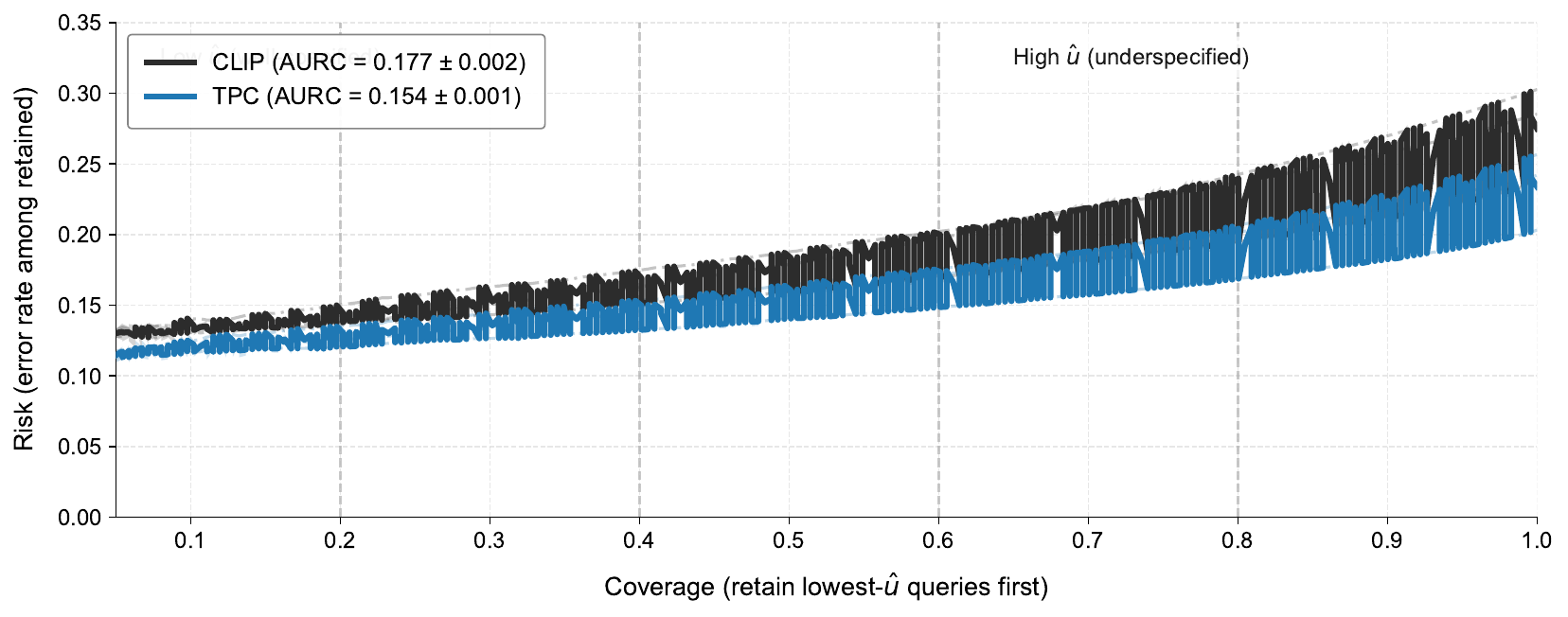}
\caption{\textbf{Risk--coverage from $\hat u$.}
Protocol in App.~\ref{app:risk_coverage_protocol}.}
\label{fig:risk_coverage_u}
\end{figure}

\section{Conclusion}
\label{sec:conclusion}

We address brittleness and overconfidence in vision--language alignment under underspecified captions, and propose \textsc{Text as Partial Constraint} (\textsc{TPC}) to align images to view-invariant semantics while suppressing view-specific residuals with single-text deployability.
Our results show that modeling language as an incomplete constraint and calibrating by view disagreement yields more reliable robustness and LVLM transfer.
Future work will scale view acquisition (generation/retrieval) and extend partial-constraint alignment to multimodal prompting and selective, safety-aware deployment.
\bibliographystyle{plainnat}
\bibliography{reference}

@inproceedings{xiao2026reversible,
  title={Reversible primitive--composition alignment for continual vision--language learning},
  author={Xiao, Canran and Xu, Tianxiang and Ma, Siyuan and Jiang, Yiyang and Gao, Haoyu and Wu, Yuhan},
  booktitle={The Fourteenth International Conference on Learning Representations},
  year={2026}
}

@inproceedings{zhang2026pi,
  title={Pi-cca: Prompt-invariant cca certificates for replay-free continual multimodal learning},
  author={Zhang, Jiayu and Zhao, Chuangxin and Xiao, Canran and Duan, Ruibo and Mo, Wenyi and Gao, Haoyu and Wang, Wenshuo},
  booktitle={The Fourteenth International Conference on Learning Representations},
  year={2026}
}

@inproceedings{zhou2026comem,
  title={Comem: Compositional concept-graph memory for vision--language adaptation},
  author={Zhou, Heng and Tang, Jing and Zhang, Jusheng and Li, Yanshu and Xiao, Canran and Hou, Liwei and Ke, Zong and Yao, Jiawei},
  booktitle={The Fourteenth International Conference on Learning Representations},
  year={2026}
}

@inproceedings{pmlr-v235-schlarmann24a,
  title={Robust CLIP: unsupervised adversarial fine-tuning of vision embeddings for robust large vision-language models},
  author={Schlarmann, Christian and Singh, Naman Deep and Croce, Francesco and Hein, Matthias},
  booktitle={Proceedings of the 41st International Conference on Machine Learning},
  pages={43685--43704},
  year={2024}
}

@inproceedings{dong2025robustsuperalignment,
  title={Robust CLIP: unsupervised adversarial fine-tuning of vision embeddings for robust large vision-language models},
  author={Schlarmann, Christian and Singh, Naman Deep and Croce, Francesco and Hein, Matthias},
  booktitle={Proceedings of the 41st International Conference on Machine Learning},
  pages={43685--43704},
  year={2024}
}

@inproceedings{pmlr-v139-jia21b,
  title={Scaling up visual and vision-language representation learning with noisy text supervision},
  author={Jia, Chao and Yang, Yinfei and Xia, Ye and Chen, Yi-Ting and Parekh, Zarana and Pham, Hieu and Le, Quoc and Sung, Yun-Hsuan and Li, Zhen and Duerig, Tom},
  booktitle={International conference on machine learning},
  pages={4904--4916},
  year={2021},
  organization={PMLR}
}

@inproceedings{leng2024vcd,
  title={Mitigating object hallucinations in large vision-language models through visual contrastive decoding},
  author={Leng, Sicong and Zhang, Hang and Chen, Guanzheng and Li, Xin and Lu, Shijian and Miao, Chunyan and Bing, Lidong},
  booktitle={Proceedings of the IEEE/CVF Conference on Computer Vision and Pattern Recognition},
  pages={13872--13882},
  year={2024}
}

@article{yu2024tgazsr,
  title={Text-guided attention is all you need for zero-shot robustness in vision-language models},
  author={Yu, Lu and Zhang, Haiyang and Xu, Changsheng},
  journal={Advances in Neural Information Processing Systems},
  volume={37},
  pages={96424--96448},
  year={2024}
}

@inproceedings{wang2024pmgaft,
  title={Pre-trained model guided fine-tuning for zero-shot adversarial robustness},
  author={Wang, Sibo and Zhang, Jie and Yuan, Zheng and Shan, Shiguang},
  booktitle={Proceedings of the IEEE/CVF conference on computer vision and pattern recognition},
  pages={24502--24511},
  year={2024}
}

@inproceedings{thrush_and_ross2022winoground,
  title={Winoground: Probing vision and language models for visio-linguistic compositionality},
  author={Thrush, Tristan and Jiang, Ryan and Bartolo, Max and Singh, Amanpreet and Williams, Adina and Kiela, Douwe and Ross, Candace},
  booktitle={Proceedings of the IEEE/CVF Conference on Computer Vision and Pattern Recognition},
  pages={5238--5248},
  year={2022}
}

@inproceedings{parcalabescu-etal-2022-valse,
  title={VALSE: A task-independent benchmark for vision and language models centered on linguistic phenomena},
  author={Parcalabescu, Letitia and Cafagna, Michele and Muradjan, Lilitta and Frank, Anette and Calixto, Iacer and Gatt, Albert},
  booktitle={Proceedings of the 60th Annual Meeting of the Association for Computational Linguistics (Volume 1: Long Papers)},
  pages={8253--8280},
  year={2022}
}

@article{choi2026underspecified,
  title={What Users Leave Unsaid: Under-Specified Queries Limit Vision-Language Models},
  author={Choi, Dasol and Son, Guijin and Lee, Hanwool and Kim, Minhyuk and Ko, Hyunwoo and Lim, Teabin and Eungyeol, Ahn and Kim, Jungwhan and Hong, Seunghyeok and Song, Youngsook},
  journal={arXiv preprint arXiv:2601.06165},
  year={2026}
}

@inproceedings{Alhamoud_2025_CVPR,
  title={Vision-language models do not understand negation},
  author={Alhamoud, Kumail and Alshammari, Shaden and Tian, Yonglong and Li, Guohao and Torr, Philip HS and Kim, Yoon and Ghassemi, Marzyeh},
  booktitle={Proceedings of the Computer Vision and Pattern Recognition Conference},
  pages={29612--29622},
  year={2025}
}

@article{dumpala2024sugarcrepepp,
  title={Sugarcrepe++ dataset: Vision-language model sensitivity to semantic and lexical alterations},
  author={Dumpala, Sri Harsha and Jaiswal, Aman and Shama Sastry, Chandramouli and Milios, Evangelos and Oore, Sageev and Sajjad, Hassan},
  journal={Advances in Neural Information Processing Systems},
  volume={37},
  pages={17972--18018},
  year={2024}
}

@article{hsieh2023sugarcrepe,
  title={Sugarcrepe: Fixing hackable benchmarks for vision-language compositionality},
  author={Hsieh, Cheng-Yu and Zhang, Jieyu and Ma, Zixian and Kembhavi, Aniruddha and Krishna, Ranjay},
  journal={Advances in neural information processing systems},
  volume={36},
  pages={31096--31116},
  year={2023}
}

@article{truong2025directedtokens,
  title={Directed-Tokens: A Robust Multi-Modality Alignment Approach to Large Language-Vision Models},
  author={Truong, Thanh-Dat and Tran, Huu-Thien and Son, Tran Thai and Raj, Bhiksha and Luu, Khoa},
  journal={arXiv preprint arXiv:2508.14264},
  year={2025}
}

@article{jiang2025finegrained,
  title={Analyzing Fine-Grained Alignment and Enhancing Vision Understanding in Multimodal Language Models},
  author={Jiang, Jiachen and Zhou, Jinxin and Peng, Bo and Ning, Xia and Zhu, Zhihui},
  journal={arXiv preprint arXiv:2505.17316},
  year={2025}
}

@inproceedings{Leng_2024_CVPR,
  title={Mitigating object hallucinations in large vision-language models through visual contrastive decoding},
  author={Leng, Sicong and Zhang, Hang and Chen, Guanzheng and Li, Xin and Lu, Shijian and Miao, Chunyan and Bing, Lidong},
  booktitle={Proceedings of the IEEE/CVF Conference on Computer Vision and Pattern Recognition},
  pages={13872--13882},
  year={2024}
}

@inproceedings{li-etal-2023-evaluating,
  title={Evaluating Object Hallucination in Large Vision-Language Models},
  author={Li, Yifan and Du, Yifan and Zhou, Kun and Wang, Jinpeng and Zhao, Wayne Xin and Wen, Ji-Rong},
  booktitle={Proceedings of the 2023 Conference on Empirical Methods in Natural Language Processing},
  pages={292--305},
  year={2023}
}

@article{awadalla2023openflamingo,
  title={Openflamingo: An open-source framework for training large autoregressive vision-language models},
  author={Awadalla, Anas and Gao, Irena and Gardner, Josh and Hessel, Jack and Hanafy, Yusuf and Zhu, Wanrong and Marathe, Kalyani and Bitton, Yonatan and Gadre, Samir and Sagawa, Shiori and others},
  journal={arXiv preprint arXiv:2308.01390},
  year={2023}
}

@article{liu2023llava,
  title={Visual instruction tuning},
  author={Liu, Haotian and Li, Chunyuan and Wu, Qingyang and Lee, Yong Jae},
  journal={Advances in neural information processing systems},
  volume={36},
  pages={34892--34916},
  year={2023}
}

@article{shen2025vlsae,
  title={VL-SAE: Interpreting and Enhancing Vision-Language Alignment with a Unified Concept Set},
  author={Shen, Shufan and Sun, Junshu and Huang, Qingming and Wang, Shuhui},
  journal={arXiv preprint arXiv:2510.21323},
  year={2025}
}

@inproceedings{Wang_2024_CVPR,
  title={Pre-trained model guided fine-tuning for zero-shot adversarial robustness},
  author={Wang, Sibo and Zhang, Jie and Yuan, Zheng and Shan, Shiguang},
  booktitle={Proceedings of the IEEE/CVF conference on computer vision and pattern recognition},
  pages={24502--24511},
  year={2024}
}

@inproceedings{mao2023understanding,
  title={Understanding Zero-shot Adversarial Robustness for Large-Scale Models},
  author={Mao, Chengzhi and Geng, Scott and Yang, Junfeng and Wang, Xin and Vondrick, Carl},
  booktitle={The Eleventh International Conference on Learning Representations},
  year = {2024}
}

@article{sun2023evaclip,
  title={Eva-clip: Improved training techniques for clip at scale},
  author={Sun, Quan and Fang, Yuxin and Wu, Ledell and Wang, Xinlong and Cao, Yue},
  journal={arXiv preprint arXiv:2303.15389},
  year={2023}
}

@inproceedings{zhai2023sigmoid,
  title={Sigmoid loss for language image pre-training},
  author={Zhai, Xiaohua and Mustafa, Basil and Kolesnikov, Alexander and Beyer, Lucas},
  booktitle={Proceedings of the IEEE/CVF international conference on computer vision},
  pages={11975--11986},
  year={2023}
}

@inproceedings{cherti2023reproducible,
  title={Reproducible scaling laws for contrastive language-image learning},
  author={Cherti, Mehdi and Beaumont, Romain and Wightman, Ross and Wortsman, Mitchell and Ilharco, Gabriel and Gordon, Cade and Schuhmann, Christoph and Schmidt, Ludwig and Jitsev, Jenia},
  booktitle={Proceedings of the IEEE/CVF conference on computer vision and pattern recognition},
  pages={2818--2829},
  year={2023}
}

@inproceedings{pmlr-v139-radford21a,
  title={Learning transferable visual models from natural language supervision},
  author={Radford, Alec and Kim, Jong Wook and Hallacy, Chris and Ramesh, Aditya and Goh, Gabriel and Agarwal, Sandhini and Sastry, Girish and Askell, Amanda and Mishkin, Pamela and Clark, Jack and others},
  booktitle={International conference on machine learning},
  pages={8748--8763},
  year={2021},
  organization={PmLR}
}

%%%%%%%%%%%%%%%%%%%%%%%%%%%%%%%%%%%%%%%%%%%%%%%%%%%%%%%%%%%%
\clearpage
\appendix

\section{Proofs}
\label{app:proofs}

\subsection{Full Proofs for \S\ref{sec:theory_dro}}
\label{app:theory_dro}

\begin{lemma}[Containment of observed views and radius bounds]
\label{lem:contain_radius}
Fix an instance $i$ with normalized text views $\{t_{i,k}\}_{k=1}^K\subset\mathbb{R}^d$ and consensus core
$c_i=\mathrm{norm}(\sum_{k=1}^K t_{i,k})$.
Define residuals $r_{i,k}\triangleq t_{i,k}-(t_{i,k}^\top c_i)c_i$ and radii $\delta_{i,k}\triangleq \|r_{i,k}\|_2$.
Let $\mathcal{U}_i\triangleq \mathrm{span}\{r_{i,k}:k\in[K]\}\subseteq c_i^\perp$ and $\rho_i\triangleq\max_{k\in[K]}\delta_{i,k}$.
Assume $t_{i,k}^\top c_i\ge 0$ for all $k\in[K]$.
Then:
\begin{enumerate}
\item (Containment) For every $k\in[K]$, $t_{i,k}\in \mathcal{A}_i(\rho_i)$ where $\mathcal{A}_i(\rho_i)$ is defined in Eq.~\eqref{eq:ambiguity_set}.
\item (Disagreement controls radius) With $u_i=\frac{1}{K}\sum_{k=1}^K\big(1-(t_{i,k}^\top c_i)^2\big)$, the maximal radius satisfies
\begin{equation}
u_i \;\le\; \rho_i^2 \;\le\; K u_i.
\label{eq:rho_u_relation}
\end{equation}
\end{enumerate}
\end{lemma}

\begin{proof}
\textbf{Step 1: orthogonal decomposition.}
Since $\|t_{i,k}\|_2=\|c_i\|_2=1$, define $r_{i,k}\triangleq t_{i,k}-(t_{i,k}^\top c_i)c_i$.
Then $c_i^\top r_{i,k}=c_i^\top t_{i,k}-(t_{i,k}^\top c_i)\|c_i\|_2^2=0$, hence $r_{i,k}\in c_i^\perp$ and in particular $r_{i,k}\in\mathcal{U}_i$.

\textbf{Step 2: identify the coefficient on $c_i$.}
Using Pythagoras on the orthogonal decomposition,
\[
\|t_{i,k}\|_2^2
=\|(t_{i,k}^\top c_i)c_i\|_2^2+\|r_{i,k}\|_2^2
=(t_{i,k}^\top c_i)^2+\delta_{i,k}^2.
\]
Since $\|t_{i,k}\|_2=1$, we obtain
\begin{equation}
\delta_{i,k}^2 = 1-(t_{i,k}^\top c_i)^2.
\label{eq:delta_identity}
\end{equation}
Under the assumption $t_{i,k}^\top c_i\ge 0$, we have
$t_{i,k}^\top c_i=\sqrt{1-\delta_{i,k}^2}$.

\textbf{Step 3: prove containment.}
Rewrite each view as
\[
t_{i,k}=(t_{i,k}^\top c_i)c_i+r_{i,k}
=\sqrt{1-\delta_{i,k}^2}\,c_i+r_{i,k}.
\]
Since $r_{i,k}\in\mathcal{U}_i$ and $\|r_{i,k}\|_2=\delta_{i,k}\le \rho_i$, this matches Eq.~\eqref{eq:ambiguity_set}, hence $t_{i,k}\in\mathcal{A}_i(\rho_i)$.

\textbf{Step 4: relate $\rho_i$ and $u_i$.}
By Eq.~\eqref{eq:delta_identity},
\[
u_i=\frac{1}{K}\sum_{k=1}^K\big(1-(t_{i,k}^\top c_i)^2\big)
=\frac{1}{K}\sum_{k=1}^K \delta_{i,k}^2.
\]
The left inequality in Eq.~\eqref{eq:rho_u_relation} follows since $\max_k \delta_{i,k}^2\ge \frac{1}{K}\sum_k \delta_{i,k}^2$:
\[
\rho_i^2=\max_k \delta_{i,k}^2 \;\ge\; \frac{1}{K}\sum_{k=1}^K \delta_{i,k}^2=u_i.
\]
For the right inequality, use $\max_k \delta_{i,k}^2 \le \sum_{k=1}^K \delta_{i,k}^2$:
\[
\rho_i^2=\max_k \delta_{i,k}^2 \;\le\; \sum_{k=1}^K \delta_{i,k}^2 = K u_i.
\]
This proves Eq.~\eqref{eq:rho_u_relation}.
\end{proof}

\begin{theorem}[Tight invariance bound under caption-view shift]
\label{thm:robust_similarity_full}
Fix $c\in\mathbb{R}^d$ with $\|c\|_2=1$, a subspace $\mathcal{U}\subseteq c^\perp$, and $\rho\in[0,1)$.
Let
\[
\mathcal{A}(c,\mathcal{U},\rho)=
\Big\{t\in\mathbb{R}^d:\ \|t\|_2=1,\ t=\sqrt{1-\|r\|_2^2}\,c+r,\ r\in\mathcal{U},\ \|r\|_2\le \rho\Big\}.
\]
For any $v\in\mathbb{R}^d$ with $\|v\|_2=1$ and $a=v^\top c\ge 0$, define $b=\|\mathbf{P}_{\mathcal{U}}v\|_2$.
Then
\[
\inf_{t\in\mathcal{A}(c,\mathcal{U},\rho)} v^\top t
=
a\sqrt{1-\rho^2}-\rho b.
\]
Moreover, for any monotone non-increasing $\ell:\mathbb{R}\to\mathbb{R}$,
\[
\sup_{t\in\mathcal{A}(c,\mathcal{U},\rho)}\ell(v^\top t)=\ell\!\Big(a\sqrt{1-\rho^2}-\rho b\Big).
\]
\end{theorem}

\begin{proof}
\textbf{Step 1: reduce to the residual variable.}
Take any $t\in\mathcal{A}(c,\mathcal{U},\rho)$.
By definition, there exists $r\in\mathcal{U}$ with $\|r\|_2\le \rho$ such that
$t=\sqrt{1-\|r\|_2^2}\,c+r$.
Then
\begin{align}
v^\top t
&= v^\top\!\big(\sqrt{1-\|r\|_2^2}\,c+r\big)
= (v^\top c)\sqrt{1-\|r\|_2^2}+ v^\top r.
\label{eq:vt_expand}
\end{align}
Since $r\in\mathcal{U}$ and $\mathbf{P}_{\mathcal{U}}$ is the orthogonal projector,
$v^\top r = (\mathbf{P}_{\mathcal{U}}v)^\top r$.
Denote $u\triangleq \mathbf{P}_{\mathcal{U}}v$ so that $\|u\|_2=b$.
Eq.~\eqref{eq:vt_expand} becomes
\begin{equation}
v^\top t = a\sqrt{1-\|r\|_2^2} + u^\top r,
\qquad \text{with}\quad r\in\mathcal{U},\ \|r\|_2\le \rho.
\label{eq:vt_reduce}
\end{equation}

\textbf{Step 2: minimize over the direction of $r$ for fixed radius.}
Fix any radius $s\in[0,\rho]$ and consider all $r\in\mathcal{U}$ with $\|r\|_2=s$.
By Cauchy--Schwarz,
\begin{equation}
u^\top r \ge -\|u\|_2\|r\|_2 = - b s,
\label{eq:cs}
\end{equation}
with equality achieved by choosing $r=-s\,u/\|u\|_2$ when $b>0$ (and any $r$ when $b=0$).
Plugging Eq.~\eqref{eq:cs} into Eq.~\eqref{eq:vt_reduce} yields, for fixed $s$,
\begin{equation}
\inf_{\substack{r\in\mathcal{U}\\ \|r\|_2=s}} v^\top t
=
a\sqrt{1-s^2}-bs.
\label{eq:fixed_s_value}
\end{equation}

\textbf{Step 3: minimize over the radius $s\in[0,\rho]$.}
Define the scalar function
$f(s)\triangleq a\sqrt{1-s^2}-bs$ on $[0,\rho]$.
Because $a\ge 0$ and $b\ge 0$, its derivative satisfies, for all $s\in(0,\rho)$,
\begin{equation}
f'(s)= -\frac{a s}{\sqrt{1-s^2}} - b \;<\;0.
\label{eq:derivative_negative}
\end{equation}
Thus $f$ is strictly decreasing on $[0,\rho]$, and the minimum is attained at $s=\rho$:
\begin{equation}
\min_{s\in[0,\rho]} f(s)= f(\rho)= a\sqrt{1-\rho^2}-b\rho.
\label{eq:min_at_rho}
\end{equation}
Combining Eq.~\eqref{eq:fixed_s_value} and Eq.~\eqref{eq:min_at_rho} proves
\[
\inf_{t\in\mathcal{A}(c,\mathcal{U},\rho)} v^\top t
= a\sqrt{1-\rho^2}-\rho b.
\]
Tightness follows by selecting $r^\star=-\rho\,u/\|u\|_2$ when $b>0$ (and $r^\star=0$ when $b=0$), which satisfies $\|r^\star\|_2=\rho$ and attains equality in Eq.~\eqref{eq:cs}.

\textbf{Step 4: convert worst-case similarity to worst-case loss.}
Since $\ell$ is monotone non-increasing, the maximizer of $\ell(v^\top t)$ over $t\in\mathcal{A}(c,\mathcal{U},\rho)$ is achieved at the minimizer of $v^\top t$:
\[
\sup_{t\in\mathcal{A}(c,\mathcal{U},\rho)}\ell(v^\top t)
=
\ell\!\Big(\inf_{t\in\mathcal{A}(c,\mathcal{U},\rho)} v^\top t\Big)
=
\ell\!\Big(a\sqrt{1-\rho^2}-\rho b\Big).
\]
\end{proof}

\subsection{Full Proofs for \S\ref{sec:theory_gen}}
\label{app:theory_gen}

\subsubsection{Auxiliary Lemmas}

\begin{lemma}[Second-moment bound for the averaged view noise]
\label{lem:mean_second_moment}
Let $\{\varepsilon_k\}_{k=1}^K$ be independent random vectors in $\mathbb{R}^d$ such that $\mathbb{E}[\varepsilon_k]=0$ and
$\|\varepsilon_k\|_2\le \sigma$ almost surely.
Then
\begin{equation}
\mathbb{E}\Big\|\frac{1}{K}\sum_{k=1}^K \varepsilon_k\Big\|_2^2
\;\le\;
\frac{\sigma^2}{K},
\qquad
\mathbb{E}\Big\|\frac{1}{K}\sum_{k=1}^K \varepsilon_k\Big\|_2
\;\le\;
\frac{\sigma}{\sqrt{K}}.
\label{eq:mean_second_moment}
\end{equation}
\end{lemma}
\begin{proof}
Let $\bar\varepsilon=\frac{1}{K}\sum_{k=1}^K \varepsilon_k$.
By expanding the squared norm and using independence and zero mean,
\begin{align*}
\mathbb{E}\|\bar\varepsilon\|_2^2
&=\mathbb{E}\Big\langle \frac{1}{K}\sum_{k=1}^K\varepsilon_k,\frac{1}{K}\sum_{\ell=1}^K\varepsilon_\ell\Big\rangle
=\frac{1}{K^2}\sum_{k=1}^K \mathbb{E}\|\varepsilon_k\|_2^2
+\frac{1}{K^2}\sum_{k\neq \ell}\mathbb{E}\langle \varepsilon_k,\varepsilon_\ell\rangle \\
&=\frac{1}{K^2}\sum_{k=1}^K \mathbb{E}\|\varepsilon_k\|_2^2
+\frac{1}{K^2}\sum_{k\neq \ell}\Big\langle \mathbb{E}[\varepsilon_k],\mathbb{E}[\varepsilon_\ell]\Big\rangle
\le \frac{1}{K^2}\sum_{k=1}^K \sigma^2
=\frac{\sigma^2}{K}.
\end{align*}
The second inequality follows from Jensen:
$\mathbb{E}\|\bar\varepsilon\|_2 \le \sqrt{\mathbb{E}\|\bar\varepsilon\|_2^2}\le \sigma/\sqrt{K}$.
\end{proof}

\begin{lemma}[McDiarmid concentration for the averaged noise norm]
\label{lem:mcdiarmid_vector_mean}
Let $\{\varepsilon_k\}_{k=1}^K$ be independent random vectors with $\|\varepsilon_k\|_2\le \sigma$ almost surely.
Define $g(\varepsilon_1,\dots,\varepsilon_K)\triangleq \big\|\frac{1}{K}\sum_{k=1}^K\varepsilon_k\big\|_2$.
Then for any $\delta\in(0,1)$,
\begin{equation}
\mathbb{P}\!\left(
g-\mathbb{E}[g] \ge \sigma\sqrt{\frac{2\log(1/\delta)}{K}}
\right)
\le \delta.
\label{eq:mcdiarmid}
\end{equation}
\end{lemma}
\begin{proof}
We verify bounded differences.
Let $(\varepsilon_1,\dots,\varepsilon_K)$ and $(\varepsilon_1,\dots,\varepsilon_k',\dots,\varepsilon_K)$ differ only at coordinate $k$.
Then by the reverse triangle inequality,
\begin{align*}
\big|g(\varepsilon_1,\dots,\varepsilon_K)-g(\varepsilon_1,\dots,\varepsilon_k',\dots,\varepsilon_K)\big|
&\le
\Big\|\frac{1}{K}\sum_{j=1}^K\varepsilon_j-\frac{1}{K}\Big(\varepsilon_k'+\sum_{j\neq k}\varepsilon_j\Big)\Big\|_2 \\
&=
\frac{1}{K}\|\varepsilon_k-\varepsilon_k'\|_2
\le \frac{1}{K}\big(\|\varepsilon_k\|_2+\|\varepsilon_k'\|_2\big)
\le \frac{2\sigma}{K}.
\end{align*}
Thus the bounded difference constants are $c_k=2\sigma/K$ for all $k$.
McDiarmid's inequality yields, for any $t>0$,
\[
\mathbb{P}(g-\mathbb{E}g\ge t)
\le
\exp\!\left(
-\frac{2t^2}{\sum_{k=1}^K c_k^2}
\right)
=
\exp\!\left(
-\frac{2t^2}{K\cdot (4\sigma^2/K^2)}
\right)
=
\exp\!\left(
-\frac{K t^2}{2\sigma^2}
\right).
\]
Setting $t=\sigma\sqrt{2\log(1/\delta)/K}$ gives Eq.~\eqref{eq:mcdiarmid}.
\end{proof}

\begin{lemma}[Normalization perturbation inequality]
\label{lem:normalize_perturb}
Let $\mu\in\mathbb{R}^d$ satisfy $\|\mu\|_2=1$ and let $e\in\mathbb{R}^d$ satisfy $\|e\|_2<1$.
Define $c=\mathrm{norm}(\mu+e)=(\mu+e)/\|\mu+e\|_2$.
Then
\begin{equation}
\|c-\mu\|_2
\le
\frac{2\|e\|_2}{1-\|e\|_2}.
\label{eq:norm_perturb}
\end{equation}
\end{lemma}
\begin{proof}
Let $u=\mu+e$ and note that $\|u\|_2\ge \|\mu\|_2-\|e\|_2=1-\|e\|_2$ by the triangle inequality.
Then
\begin{align*}
\|c-\mu\|_2
&=\Big\|\frac{u}{\|u\|_2}-\mu\Big\|_2
=\frac{1}{\|u\|_2}\|u-\|u\|_2\mu\|_2 \\
&=\frac{1}{\|u\|_2}\|\mu+e-\|u\|_2\mu\|_2
=\frac{1}{\|u\|_2}\|e+(1-\|u\|_2)\mu\|_2 \\
&\le \frac{1}{\|u\|_2}\Big(\|e\|_2+|1-\|u\|_2|\cdot\|\mu\|_2\Big)
= \frac{1}{\|u\|_2}\big(\|e\|_2+|1-\|u\|_2|\big).
\end{align*}
Using the reverse triangle inequality,
$|1-\|u\|_2|=|\|\mu\|_2-\|u\|_2|\le \|\mu-u\|_2=\|e\|_2$.
Therefore,
\[
\|c-\mu\|_2
\le \frac{1}{\|u\|_2}(2\|e\|_2)
\le \frac{2\|e\|_2}{1-\|e\|_2},
\]
which proves Eq.~\eqref{eq:norm_perturb}.
\end{proof}

\begin{lemma}[Lipschitz transfer from core error to loss error]
\label{lem:lipschitz_transfer}
Let $\ell:[-1,1]\to\mathbb{R}$ be $L$-Lipschitz.
For any $v\in\mathbb{S}^{d-1}$ and any $a,b\in\mathbb{S}^{d-1}$,
\begin{equation}
\big|\ell(v^\top a)-\ell(v^\top b)\big|
\le L\,\|a-b\|_2.
\label{eq:lipschitz_transfer}
\end{equation}
\end{lemma}
\begin{proof}
Since $\|v\|_2=1$, Cauchy--Schwarz gives $|v^\top a - v^\top b|=|v^\top(a-b)|\le \|a-b\|_2$.
By Lipschitzness of $\ell$,
$|\ell(v^\top a)-\ell(v^\top b)|\le L|v^\top a-v^\top b|\le L\|a-b\|_2$.
\end{proof}

\begin{lemma}[Rademacher contraction for inner-product losses]
\label{lem:contraction}
Let $\ell:[-1,1]\to\mathbb{R}$ be $L$-Lipschitz and let $\mathcal{F}\subseteq\{f:\mathcal{X}\to\mathbb{S}^{d-1}\}$.
For any fixed sample $\{(x_i,c_i)\}_{i=1}^n$ with $\|c_i\|_2=1$,
\begin{equation}
\mathfrak{R}_n(\ell\circ\mathcal{F})
\triangleq
\mathbb{E}_{\epsilon}\Big[\sup_{f\in\mathcal{F}}\frac{1}{n}\sum_{i=1}^n \epsilon_i\,\ell\big(f(x_i)^\top c_i\big)\Big]
\le
L\,\mathfrak{R}_n(\mathcal{F}),
\label{eq:contraction}
\end{equation}
where $\mathfrak{R}_n(\mathcal{F})=\mathbb{E}_{\epsilon}\big[\sup_{f\in\mathcal{F}}\frac{1}{n}\sum_{i=1}^n \epsilon_i f(x_i)^\top c_i\big]$.
\end{lemma}
\begin{proof}
This is a direct application of the Ledoux--Talagrand contraction principle.
For completeness, let $\phi_i(s)=\ell(s)$.
Each $\phi_i$ is $L$-Lipschitz on $[-1,1]$ and $\phi_i(0)$ is constant, hence
\[
\mathbb{E}_{\epsilon}\Big[\sup_{f\in\mathcal{F}}\frac{1}{n}\sum_{i=1}^n \epsilon_i\,\phi_i(f(x_i)^\top c_i)\Big]
\le
L\,\mathbb{E}_{\epsilon}\Big[\sup_{f\in\mathcal{F}}\frac{1}{n}\sum_{i=1}^n \epsilon_i\,f(x_i)^\top c_i\Big],
\]
which is Eq.~\eqref{eq:contraction}.
\end{proof}

\subsubsection{Proof of Theorem~\ref{thm:pac_denoise}}

\begin{proof}
We split the proof into four steps corresponding to the proof sketch.

\textbf{Step 1: uniform bound on the averaged noise.}
Fix $i\in[n]$ and write $\bar\varepsilon_i=\frac{1}{K}\sum_{k=1}^K \varepsilon_{i,k}$.
By Lemma~\ref{lem:mcdiarmid_vector_mean}, for any $\delta_i\in(0,1)$,
\begin{equation}
\mathbb{P}\!\left(
\|\bar\varepsilon_i\|_2 \ge \mathbb{E}\|\bar\varepsilon_i\|_2 + \sigma\sqrt{\frac{2\log(1/\delta_i)}{K}}
\right)\le \delta_i.
\label{eq:bar_eps_tail}
\end{equation}
By Lemma~\ref{lem:mean_second_moment}, $\mathbb{E}\|\bar\varepsilon_i\|_2\le \sigma/\sqrt{K}$.
Substituting into Eq.~\eqref{eq:bar_eps_tail} gives
\begin{equation}
\mathbb{P}\!\left(
\|\bar\varepsilon_i\|_2 \ge \frac{\sigma}{\sqrt{K}} + \sigma\sqrt{\frac{2\log(1/\delta_i)}{K}}
\right)\le \delta_i.
\label{eq:bar_eps_tail2}
\end{equation}
Set $\delta_i=\delta/(2n)$ and apply a union bound over $i\in[n]$:
\begin{align}
\mathbb{P}\!\left(
\max_{i\in[n]}\|\bar\varepsilon_i\|_2 \ge \frac{\sigma}{\sqrt{K}} + \sigma\sqrt{\frac{2\log(2n/\delta)}{K}}
\right)
&\le \sum_{i=1}^n \mathbb{P}\!\left(
\|\bar\varepsilon_i\|_2 \ge \frac{\sigma}{\sqrt{K}} + \sigma\sqrt{\frac{2\log(2n/\delta)}{K}}
\right)\notag\\
&\le \sum_{i=1}^n \frac{\delta}{2n}
=\frac{\delta}{2}.
\label{eq:uniform_bar_eps}
\end{align}
Define
\begin{equation}
\eta_{K,n}(\delta)\triangleq \frac{\sigma}{\sqrt{K}} + \sigma\sqrt{\frac{2\log(2n/\delta)}{K}}.
\label{eq:eta_def}
\end{equation}
Then Eq.~\eqref{eq:uniform_bar_eps} states that, with probability at least $1-\delta/2$,
\begin{equation}
\max_{i\in[n]}\|\bar\varepsilon_i\|_2 \le \eta_{K,n}(\delta).
\label{eq:uniform_eta}
\end{equation}

\textbf{Step 2: uniform denoising bound on consensus cores.}
Recall $\bar t_i=\mu_i+\bar\varepsilon_i$ and $c_i=\mathrm{norm}(\bar t_i)$.
On the event in Eq.~\eqref{eq:uniform_eta}, Lemma~\ref{lem:normalize_perturb} with $e=\bar\varepsilon_i$ implies, for every $i\in[n]$,
\begin{equation}
\|c_i-\mu_i\|_2
\le \frac{2\|\bar\varepsilon_i\|_2}{1-\|\bar\varepsilon_i\|_2}
\le \frac{2\eta_{K,n}(\delta)}{1-\eta_{K,n}(\delta)}
\triangleq \varepsilon_{K,n}(\delta).
\label{eq:uniform_core_error}
\end{equation}
Thus, with probability at least $1-\delta/2$,
\begin{equation}
\max_{i\in[n]} \|c_i-\mu_i\|_2 \le \varepsilon_{K,n}(\delta).
\label{eq:uniform_core_error2}
\end{equation}

\textbf{Step 3: transfer from observed risk to true risk.}
Fix any $f\in\mathcal{F}$.
For each $i\in[n]$, Lemma~\ref{lem:lipschitz_transfer} with $v=f(x_i)$, $a=c_i$, and $b=\mu_i$ yields
\begin{equation}
\big|\ell(f(x_i)^\top c_i)-\ell(f(x_i)^\top \mu_i)\big|
\le L\,\|c_i-\mu_i\|_2.
\label{eq:per_sample_transfer}
\end{equation}
On the event in Eq.~\eqref{eq:uniform_core_error2}, Eq.~\eqref{eq:per_sample_transfer} implies
\begin{equation}
\ell(f(x_i)^\top \mu_i)
\le \ell(f(x_i)^\top c_i) + L\,\varepsilon_{K,n}(\delta),
\qquad \forall i\in[n].
\label{eq:per_sample_transfer2}
\end{equation}
Averaging over $i$ gives the empirical relation
\begin{equation}
\widehat R_\mu(f)\triangleq \frac{1}{n}\sum_{i=1}^n \ell(f(x_i)^\top \mu_i)
\le \widehat R_c(f)+L\,\varepsilon_{K,n}(\delta).
\label{eq:empirical_transfer}
\end{equation}
Taking expectation over fresh draws (and using the same Lipschitz argument pointwise) yields the population relation
\begin{equation}
R_\mu(f)\le R_c(f)+L\,\varepsilon_{K,n}(\delta).
\label{eq:population_transfer}
\end{equation}

\textbf{Step 4: PAC generalization on the observed risk and ERM decomposition.}
Apply a standard Rademacher generalization bound to the class $\ell\circ\mathcal{F}$ on the sample $\{(x_i,c_i)\}_{i=1}^n$:
with probability at least $1-\delta/2$,
\begin{equation}
\sup_{f\in\mathcal{F}}\big(R_c(f)-\widehat R_c(f)\big)
\le 2\,\mathfrak{R}_n(\ell\circ\mathcal{F})+\sqrt{\frac{\log(4/\delta)}{2n}}.
\label{eq:rad_gen}
\end{equation}
By Lemma~\ref{lem:contraction}, $\mathfrak{R}_n(\ell\circ\mathcal{F})\le L\,\mathfrak{R}_n(\mathcal{F})$, hence
\begin{equation}
\sup_{f\in\mathcal{F}}\big(R_c(f)-\widehat R_c(f)\big)
\le 2L\,\mathfrak{R}_n(\mathcal{F})+\sqrt{\frac{\log(4/\delta)}{2n}}.
\label{eq:rad_gen2}
\end{equation}
On the intersection of events \eqref{eq:uniform_core_error2} and \eqref{eq:rad_gen2} (which holds with probability at least $1-\delta$ by a union bound),
we bound the excess true risk of $\widehat f$.

Let $f^\star\in\arg\min_{f\in\mathcal{F}} R_\mu(f)$.
Starting from Eq.~\eqref{eq:population_transfer} and using Eq.~\eqref{eq:rad_gen2} twice:
\begin{align}
R_\mu(\widehat f)
&\le R_c(\widehat f)+L\,\varepsilon_{K,n}(\delta)
\label{eq:chain1}\\
&\le \widehat R_c(\widehat f)+\Big(2L\,\mathfrak{R}_n(\mathcal{F})+\sqrt{\tfrac{\log(4/\delta)}{2n}}\Big)+L\,\varepsilon_{K,n}(\delta)
\label{eq:chain2}\\
&\le \widehat R_c(f^\star)+\Big(2L\,\mathfrak{R}_n(\mathcal{F})+\sqrt{\tfrac{\log(4/\delta)}{2n}}\Big)+L\,\varepsilon_{K,n}(\delta)
\label{eq:chain3}\\
&\le R_c(f^\star)+2\Big(2L\,\mathfrak{R}_n(\mathcal{F})+\sqrt{\tfrac{\log(4/\delta)}{2n}}\Big)+L\,\varepsilon_{K,n}(\delta)
\label{eq:chain4}\\
&\le R_\mu(f^\star)+2\Big(2L\,\mathfrak{R}_n(\mathcal{F})+\sqrt{\tfrac{\log(4/\delta)}{2n}}\Big)+2L\,\varepsilon_{K,n}(\delta),
\label{eq:chain5}
\end{align}
where \eqref{eq:chain3} uses ERM optimality of $\widehat f$ on $\widehat R_c$, and \eqref{eq:chain5} uses Eq.~\eqref{eq:population_transfer} for $f^\star$.
Rearranging yields Eq.~\eqref{eq:pac_denoise_bound}.

Finally, if $\eta_{K,n}(\delta)\le \tfrac{1}{2}$, then
$\varepsilon_{K,n}(\delta)=\frac{2\eta}{1-\eta}\le 4\eta$, giving the stated $O(\sigma\sqrt{\log(n/\delta)/K})$ rate.
\end{proof}

\subsubsection{Proof of Corollary~\ref{cor:k_sample_complexity}}
\label{app:cora}
\begin{proof}
From Theorem~\ref{thm:pac_denoise}, it suffices to enforce $2L\,\varepsilon_{K,n}(\delta)\le \epsilon$.
When $\eta_{K,n}(\delta)\le \tfrac{1}{2}$, $\varepsilon_{K,n}(\delta)\le 4\eta_{K,n}(\delta)$, so it suffices that
$8L\,\eta_{K,n}(\delta)\le \epsilon$.
By the definition of $\eta_{K,n}(\delta)$ in Eq.~\eqref{eq:eta_def}, this holds when
\[
\frac{\sigma}{\sqrt{K}}+\sigma\sqrt{\frac{2\log(2n/\delta)}{K}}
\le \frac{\epsilon}{8L},
\]
which is implied by $K=\Omega(\sigma^2\log(n/\delta)/\epsilon^2)$ (absorbing constants and $L$ into $\Omega(\cdot)$).
\end{proof}

\section{Additional Experimental Details}
\label{app:exp_details}

\subsection{Preliminary Study Protocol}
\label{app:prestudy_protocol}

This subsection provides the full protocol for the preliminary study in Sec.~\ref{sec:preliminary_study}.
The study is purely diagnostic.

\textbf{Data and encoder.}
We use multi-caption image--text datasets such as MS-COCO and Flickr30K, where each image $x_i$ is paired with $K$ captions
$\{y_{i,k}\}_{k=1}^{K}$.
Unless stated otherwise, we use all available human captions and set $K=5$.
A frozen dual encoder maps images and captions to normalized embeddings:
\begin{equation}
v_i=\mathrm{norm}(\phi_v(x_i))\in\mathbb{R}^d,
\qquad
t_{i,k}=\mathrm{norm}(\phi_t(y_{i,k}))\in\mathbb{R}^d.
\end{equation}
All image--text similarities are cosine similarities $v^\top t$.

\textbf{Caption-view dispersion.}
We quantify how much the captions of the same image disagree in the frozen text embedding space:
\begin{equation}
d_i^{\mathrm{view}}
=
\frac{2}{K(K-1)}
\sum_{1\le k<\ell\le K}
\bigl(1-t_{i,k}^{\top}t_{i,\ell}\bigr).
\label{eq:app_view_dispersion}
\end{equation}
A larger $d_i^{\mathrm{view}}$ means that different valid captions share less overlap and leave more details unspecified.

\textbf{Caption-induced rank volatility.}
For each caption $y_{i,k}$, we retrieve images from a fixed gallery $\mathcal{G}$ and record the rank of the ground-truth image:
\begin{equation}
R_{i,k}
=
1+
\sum_{x_j\in\mathcal{G},\,j\neq i}
\mathbf{1}
\!\left[
v_j^{\top}t_{i,k}>v_i^{\top}t_{i,k}
\right].
\label{eq:app_caption_rank}
\end{equation}
We then define the rank volatility across captions as
\begin{equation}
\sigma_i^{\mathrm{rank}}
=
\mathrm{Std}_{k\in[K]}
\bigl(\log(1+R_{i,k})\bigr).
\label{eq:app_rank_volatility}
\end{equation}
The logarithm reduces the domination of extreme ranks while preserving caption-induced instability.

\textbf{Hard-negative posterior mass.}
To measure whether caption disagreement makes negatives more competitive, we compute a retrieval posterior over the same gallery:
\begin{equation}
p(j\mid y_{i,k})
=
\frac{\exp(v_j^\top t_{i,k}/\tau_0)}
{\sum_{m:x_m\in\mathcal{G}}\exp(v_m^\top t_{i,k}/\tau_0)} ,
\label{eq:app_retrieval_posterior}
\end{equation}
where $\tau_0$ is the frozen encoder's evaluation temperature.
The hard-negative mass for image $i$ is
\begin{equation}
p_i^{\mathrm{hard}}
=
\frac{1}{K}\sum_{k=1}^{K}
\max_{j\neq i}p(j\mid y_{i,k}).
\label{eq:app_hard_negative_mass}
\end{equation}
A larger value indicates that at least one incorrect image receives high posterior probability under the caption query.

\textbf{Diagnostic residual leakage.}
We next examine whether the image embedding aligns with view-specific textual components.
We first define a diagnostic multi-caption centroid
\begin{equation}
\tilde t_i
=
\mathrm{norm}
\!\left(
\sum_{k=1}^{K}t_{i,k}
\right),
\label{eq:app_diagnostic_centroid}
\end{equation}
which is used only for analysis.
For each caption, we remove the component parallel to this centroid:
\begin{equation}
q_{i,k}
=
t_{i,k}
-
(t_{i,k}^{\top}\tilde t_i)\tilde t_i .
\label{eq:app_diagnostic_residual}
\end{equation}
The residual leakage score is
\begin{equation}
\ell_{i,k}^{\mathrm{res}}
=
\bigl(v_i^{\top}q_{i,k}\bigr)^2 .
\label{eq:app_residual_leakage}
\end{equation}
Large $\ell_{i,k}^{\mathrm{res}}$ means that the image embedding is correlated with the caption component that is not shared by the multi-caption centroid.

\textbf{High-confidence retrieval error.}
For each caption query, we define the top prediction and its confidence as
\begin{equation}
\hat j_{i,k}
=
\arg\max_{j:x_j\in\mathcal{G}}p(j\mid y_{i,k}),
\qquad
\pi_{i,k}
=
\max_{j:x_j\in\mathcal{G}}p(j\mid y_{i,k}).
\end{equation}
The high-confidence error score is
\begin{equation}
o_{i,k}
=
\pi_{i,k}
\cdot
\mathbf{1}\!\left[\hat j_{i,k}\neq i\right].
\label{eq:app_confident_error}
\end{equation}
This score is large only when the frozen model retrieves a wrong image with high confidence.

\textbf{Visualization.}
Fig.~\ref{fig:prestudy_failure_mechanism}(a) plots each image as one point, with $x$-axis $d_i^{\mathrm{view}}$ and $y$-axis $\sigma_i^{\mathrm{rank}}$.
Point color encodes $p_i^{\mathrm{hard}}$, so the plot jointly shows caption disagreement, rank instability, and hard-negative confusion.
Fig.~\ref{fig:prestudy_failure_mechanism}(b) plots each caption view as one point, with $x$-axis $\ell_{i,k}^{\mathrm{res}}$ and $y$-axis $o_{i,k}$.
Point color encodes the parent image's caption-view dispersion $d_i^{\mathrm{view}}$.
For both panels, we overlay binned means and upper quantiles to show the global trend without hiding the dense point distribution.

\textbf{Interpretation.}
The preliminary study is designed to reveal a failure mechanism rather than to evaluate \textsc{TPC}.
It supports the following empirical pattern:
caption disagreement is associated with unstable retrieval ranks, and residual leakage is associated with confident errors.
These two observations motivate the method design in Sec.~\ref{sec:method}: align to a view-shared semantic component, learn a single-caption estimator for deployment, suppress view-specific residual alignment, and calibrate confidence when captions provide weak constraints.

\subsection{Datasets and Protocols}
\label{app:datasets_protocols}

\textbf{Training data and multi-view construction.}
\textsc{Text as Partial Constraint} requires $K$ textual views per image to instantiate a \emph{partial} language constraint and expose an ``unsaid'' subspace.
We therefore use a \emph{two-tier} view construction strategy:

\emph{(A) Human multi-caption datasets.}
For datasets that naturally provide multiple captions, we use \textbf{all available human captions} and set $K{=}5$ (MS-COCO, Flickr30K).
During training, we encode each image \emph{once} and encode its $K$ captions independently, so multi-view supervision increases text-side compute but does not multiply vision-side cost.
To match the step budget of single-caption baselines, we subsample $K' \in \{3,4,5\}$ captions per step (uniformly without replacement) and cycle through captions across epochs; the consensus $c_i$ and uncertainty $u_i$ are always computed on the sampled set.

\emph{(B) Single-caption web-style corpora.}
When only one caption is available, we synthesize view diversity that is faithful but intentionally underspecified.
Given the original caption $y$, we generate $K{-}1$ additional views consisting of:
(i) a \emph{core-only} view that retains only the main entities/actions (drops attributes such as colors, counts, and fine descriptors),
(ii) a \emph{paraphrase} view with similar semantics but different surface form,
and (iii) a \emph{back-translation} view (EN$\rightarrow$DE$\rightarrow$EN) to inject lexical variation.
Unless stated otherwise, we set $K{=}4$ for this regime (original + three synthesized views).
We filter synthesized views to avoid semantic drift by enforcing: (1) length in $[5,30]$ tokens,
(2) no new named entities compared to $y$,
and (3) cosine similarity to the original caption above a threshold under a \emph{frozen} text encoder (we use $\ge 0.25$ as default).
If a view fails, we resample it; if repeated failures occur, we fall back to another core-only deletion sample to preserve correctness.
This construction produces controlled omissions, which is precisely the regime where our residual suppression and uncertainty-aware sharpness are expected to help.

\textbf{Image--text retrieval benchmarks.}
We evaluate retrieval on MS-COCO and Flickr30K using standard splits.
At test time, each query caption $y$ is mapped to the single-view core $\hat c=\mathrm{norm}(\psi_\theta(\mathrm{norm}(\phi_t(y))))$ and each candidate image is mapped to $v=\mathrm{norm}(\phi_v(x))$; ranking is performed by the core similarity $s(x,y)=v^\top \hat c$.
We report image-to-text and text-to-image Recall@\{1,5,10\}.
To stress robustness to underspecification, we additionally report:
(i) \textbf{Mean-over-captions} performance (average Recall@K over all captions of an image),
(ii) \textbf{Worst-caption} performance (minimum over captions),
and (iii) \textbf{Caption sensitivity} (rank standard deviation across captions), see Appendix~\S\ref{app:exp_metrics}.
These explicitly measure whether the model ``over-commits'' to view-specific details.

\textbf{Zero-shot classification and distribution shift.}
We follow a CLIP-style zero-shot protocol with prompt ensembling.
For ImageNet-1K (val), we use a fixed set of templates (e.g., the standard CLIP template set) and average the normalized text embeddings across templates per class.
For transfer, we use dataset-specific template sets when available and otherwise use a small shared template pool (e.g., 8 generic templates).
Crucially, our method uses the \emph{core} text embedding for classification, i.e., class scores are $v^\top \hat c_{\text{class}}$.
We report Top-1 accuracy on clean images and, when enabled, robust accuracy under AutoAttack (\S\ref{app:exp_metrics}).
We also include fine-grained alignment diagnostics that probe attribute/relation sensitivity, which closely matches our ``unsaid'' motivation:
Winoground (compositional correctness), ARO-style attribute/object/relation tests, and caption-contrast sets (e.g., hard negatives formed by swapping attributes/relations).
These are reported as accuracy (higher is better).

\textbf{LVLM downstream and hallucination benchmarks.}
To test whether improved alignment transfers to generation settings, we plug our vision encoder into a fixed LVLM stack.
Concretely, we keep the LLM weights, tokenizer, and instruction-tuning data fixed, and only change the vision encoder.
Because changing the vision encoder can alter feature statistics/dimensions, we retrain (or re-fit) only the vision--language projector for a small number of steps while freezing the LLM (and optionally freezing the vision encoder after a short alignment warmup).
We evaluate: (i) VQAv2 and TextVQA (accuracy) and (ii) POPE (accuracy/precision/recall/F1) for object hallucination.
For open-ended caption hallucination, we report CHAIR.
We present LVLM results separately from dual-encoder retrieval/classification to avoid conflating architectural and training differences.

\subsection{Evaluation Metrics}
\label{app:exp_metrics}

\textbf{Retrieval and caption robustness.}
We report Recall@\{1,5,10\} for both directions.
For caption robustness, given an image with captions $\{y_k\}_{k=1}^{K}$ we compute:
(i) mean Recall@K over captions,
(ii) worst-caption Recall@K,
and (iii) caption-sensitivity as the standard deviation of the (1-indexed) rank of the ground-truth match across captions.
We additionally report a \textbf{disagreement-to-error} analysis: we bucket queries by predicted residual energy
$\hat u(y)=1-(t^\top \hat c)^2$ and report Recall@K and calibration within each bucket, which directly tests whether our uncertainty signal is meaningful.

\textbf{Calibration for retrieval and classification.}
For classification, we compute ECE with $M{=}15$ equal-width bins over top-1 confidence and report NLL.
For retrieval, we convert similarities into a calibrated distribution using our per-query temperature $\tau(y)$.
Because a full softmax over the entire gallery can be expensive, we compute calibration on a \textbf{standardized candidate set}:
for each query we include the ground-truth match, the top-$N$ retrieved candidates (default $N{=}100$), and a fixed number of uniformly sampled negatives (default 300), then apply softmax on this set.
We report ECE/NLL on top-1 correctness, Brier score, and risk--coverage curves (selective retrieval) obtained by abstaining when max softmax probability is below a threshold.

\textbf{Adversarial robustness .}
When comparing to adversarially trained or adversarially fine-tuned baselines, we evaluate robust Top-1 accuracy on ImageNet-1K under AutoAttack with an $\ell_\infty$ threat model and a fixed perturbation radius.
We use a single, fixed radius across all methods (default $\epsilon{=}4/255$) and the same attack settings (steps, restarts) without per-method tuning.

\subsection{Implementation Details}
\label{app:impl_details}

\textbf{Backbone, resolution, and tokenization.}
We use an OpenCLIP-style dual encoder with a ViT vision backbone and a Transformer text encoder.
Unless stated otherwise, images are processed at $224{\times}224$ with standard CLIP augmentations (random resized crop and horizontal flip for training; center crop for evaluation).
Texts are truncated/padded to a fixed maximum length (77 tokens in the CLIP tokenizer) and encoded with the backbone tokenizer.
We report results for two model scales: ViT-B/16 for fast ablations and ViT-L/14 for main results.

\textbf{Core filter initialization and stability.}
The single-view core filter $\psi_\theta(t)=W_ct+b_c$ is initialized to be near-identity ($W_c \leftarrow I$ and $b_c \leftarrow 0$), so $\hat c$ starts close to the original text direction and training focuses on \emph{removing} inconsistent components rather than rotating the space arbitrarily.
We clamp the instance temperature to avoid extremes, i.e., $\tau_i=\tau_0(1+\gamma u_i)$ is clipped to $[\tau_0,\, 3\tau_0]$ by default, which stabilizes early training when disagreement estimates are noisy.

\textbf{Optimization and schedule.}
We use AdamW with cosine learning-rate decay and linear warmup.
A typical configuration is: global batch size 4096 (via gradient accumulation if needed), warmup 5\% of total steps, and total 100k--200k steps depending on data scale.
We use mixed precision (bf16) and gradient clipping (max norm 1.0).
We set separate learning rates for (i) encoders and (ii) the core filter: $\text{LR}_{\text{enc}}=5\times10^{-5}$ and $\text{LR}_{\psi}=5\times10^{-4}$ by default, with weight decay $0.1$.
Unless otherwise stated, we tune only $\text{LR}_{\text{enc}}$ and keep $\text{LR}_{\psi}$ fixed.

\textbf{Method hyperparameters.}
We use $K{=}5$ on COCO/Flickr30K and $K{=}4$ on synthesized-view corpora.
Unless otherwise stated, we set $\tau_0{=}0.07$, $\gamma{=}1.0$, $\lambda_{\mathrm{agree}}{=}1.0$, and $\lambda_{\mathrm{nc}}{=}0.1$.
We tune $\gamma\in\{0,0.5,1,2\}$, $\lambda_{\mathrm{agree}}\in\{0.5,1,2\}$, and $\lambda_{\mathrm{nc}}\in\{0.05,0.1,0.2,0.5\}$ on a held-out validation split.
For ablations, we also report: (i) removing uncertainty scaling ($\gamma{=}0$), (ii) removing residual suppression ($\lambda_{\mathrm{nc}}{=}0$), and (iii) using consensus-only training without the single-view filter.

\textbf{Compute and reproducibility.}
Training is run with distributed data parallel on A100-class GPUs.
We match global batch size and total optimization steps across methods whenever feasible. For multi-view training, the image forward pass is shared across views to keep compute comparable.
We run each main setting with three random seeds and report mean$\pm$std.

\subsection{Compared Methods}
\label{app:baselines}

\textbf{CLIP-style contrastive alignment.}
We use the standard CLIP objective as the primary baseline~\citep{pmlr-v139-radford21a}, i.e., symmetric InfoNCE over image--text pairs with cosine similarity.
We also include OpenCLIP~\citep{cherti2023reproducible} as an open, large-scale reproduction with widely used checkpoints and evaluation protocols.

\textbf{Stronger contemporary training recipes.}
We compare to SigLIP~\citep{zhai2023sigmoid}, which replaces softmax-normalized contrastive loss with a pairwise sigmoid objective, and EVA-CLIP~\citep{sun2023evaclip}, which introduces improved scaling/training techniques for CLIP at scale.
When these baselines are used as \emph{off-the-shelf} checkpoints trained on different corpora, we report them as ``external checkpoints'' under the same evaluation suite; when data access permits, we additionally retrain with our data/budget and report the matched-data results.

\textbf{Robust/anti-misalignment training and adaptation.}
We include TeCoA~\citep{mao2023understanding}, which adversarially fine-tunes CLIP-style models (min--max optimization) to improve zero-shot adversarial robustness under distribution shifts.
We include RobustCLIP/FARE~\citep{pmlr-v235-schlarmann24a}, which performs \emph{unsupervised} adversarial fine-tuning of the \emph{vision embedding} while explicitly preserving the original CLIP features, enabling plug-and-play robustness gains in downstream LVLMs.
We include Robust SuperAlignment (Adv-W2S)~\citep{dong2025robustsuperalignment}, which extends weak-to-strong generalization to \emph{robustness transfer} by incorporating adversarial examples into the alignment objective.
These methods incur additional adversarial-example generation and/or teacher guidance; we therefore (i) report them in a separate robustness block, (ii) keep the backbone and evaluation protocol identical, and (iii) match training steps/compute as closely as possible.

\textbf{Post-hoc concept-level enhancement.}
We compare against VL-SAE~\citep{shen2025vlsae}, which learns a unified concept set via sparse autoencoding of vision--language representations and can be used to strengthen alignment at the concept level.
We apply VL-SAE as a post-hoc module on top of the same frozen encoders (no encoder retraining), isolating the effect of the enhancement mechanism.

\textbf{LVLM-side alignment baselines.}
To test transfer to generative multimodal models, we follow the plug-and-play LVLM protocol: we fix the LVLM framework (e.g., LLaVA~\citep{liu2023llava} or OpenFlamingo~\citep{awadalla2023openflamingo}) and only swap the vision encoder (ours vs.\ baseline encoders).
We further compare against LVLM alignment improvements that modify the connector/training, including patch-aligned training~\citep{jiang2025finegrained} and Directed-Tokens~\citep{truong2025directedtokens}.
LVLM-side methods are not directly comparable to dual-encoder retrieval training; we therefore report LVLM results separately and keep the LLM, connector architecture, and SFT data identical unless the baseline explicitly changes them.

\subsection{Main Experiment Details}
\label{app:experimental_details}

\paragraph{Residual Leakage Diagnostic.}
\label{app:residual_leakage_protocol}
This diagnostic evaluates whether training changes the degree to which image embeddings align with caption-specific components that are not shared across views.
For each image $x_i$ with $K$ captions $\{y_{i,k}\}_{k=1}^{K}$, we compute normalized embeddings
\begin{equation}
v_i=\mathrm{norm}(\phi_v(x_i)),
\qquad
t_{i,k}=\mathrm{norm}(\phi_t(y_{i,k})) .
\end{equation}
We define a caption-consensus core using all available caption views:
\begin{equation}
c_i=\mathrm{norm}\!\left(\sum_{k=1}^{K} t_{i,k}\right).
\end{equation}
For each caption view, we remove the component parallel to the consensus core and obtain the diagnostic residual
\begin{equation}
r_{i,k}
=
t_{i,k}-(t_{i,k}^{\top}c_i)c_i .
\end{equation}
The residual leakage score is then
\begin{equation}
\ell_{i,k}
=
(v_i^{\top}r_{i,k})^2 .
\end{equation}
A larger $\ell_{i,k}$ means that the image representation is more strongly correlated with a view-specific textual component rather than the shared semantic core.
We compute $\ell_{i,k}$ before and after fine-tuning for standard CLIP and \textsc{TPC}.
For a fair comparison, the residual space is always defined by the multi-caption consensus $c_i$ rather than by the learned single-view core predictor $\psi_\theta$.
Fig.~\ref{fig:residual_leakage} visualizes the distribution over all image--caption pairs, with mean and 95\% confidence intervals estimated across seeds and dataset-wise means shown as small markers.

\paragraph{Selective Risk--Coverage from \texorpdfstring{$\hat u$}{u}.}
\label{app:risk_coverage_protocol}
This diagnostic evaluates whether the uncertainty proxy $\hat u$ identifies queries whose text constraints are weak.
Given a query caption $y$, we compute
\begin{equation}
t=\mathrm{norm}(\phi_t(y)),
\qquad
\hat c=\mathrm{norm}(\psi_\theta(t)),
\qquad
\hat u(y)=1-(t^\top \hat c)^2 .
\end{equation}
Intuitively, $\hat u(y)$ is large when the original text embedding contains substantial energy outside the predicted core direction, indicating that the query may contain view-specific or underspecified information.
Let $\mathcal{Q}$ denote the evaluation query set and let $e_m(q)\in\{0,1\}$ be the top-1 error of method $m$ on query $q$.
For threshold $\theta$, we retain the low-uncertainty subset
\begin{equation}
\mathcal{Q}_\theta
=
\{q\in\mathcal{Q}\mid \hat u(q)\le \theta\}.
\end{equation}
We then compute coverage and method-specific risk as
\begin{equation}
\mathrm{coverage}(\theta)
=
\frac{|\mathcal{Q}_\theta|}{|\mathcal{Q}|},
\qquad
\mathrm{risk}_m(\theta)
=
\frac{1}{|\mathcal{Q}_\theta|}
\sum_{q\in\mathcal{Q}_\theta} e_m(q).
\end{equation}
Sweeping $\theta$ from small to large yields a risk--coverage curve.
Lower risk at the same coverage indicates that the retained low-$\hat u$ queries are indeed more reliable.
In Fig.~\ref{fig:risk_coverage_u}, all methods are evaluated on the same retained query subsets induced by $\textsc{TPC}$'s $\hat u$, so differences reflect prediction reliability under the same abstention policy.
Solid curves report the overall average, faint curves report per-dataset trends, and shaded bands indicate $\pm1$ standard deviation over three seeds.

\subsection{Compute Resources}
\label{app:compute_resources}

All experiments are run with distributed data parallel training on NVIDIA A100-class GPUs. Unless otherwise stated, training uses bf16 mixed precision, gradient accumulation to maintain the stated global batch size, and shared image forward passes across caption views. Each worker has 80GB GPU memory, 64 CPU cores, approximately 512GB host memory, and local NVMe storage for dataset caching. The main software stack uses PyTorch, CUDA, NCCL, OpenCLIP-style data loading, and deterministic evaluation scripts.

Table~\ref{tab:compute_resources} reports the approximate compute required to reproduce the reported experiments. GPU-hours are computed as
\[
\text{GPU-hours} = \#\text{GPUs} \times \text{wall-clock hours}.
\]
The total compute for experiments reported in the paper is approximately 2.8K--3.2K A100 GPU-hours, including main training, ablations, sensitivity runs, robustness evaluation, LVLM projector fitting, and controlled-deletion evaluation. Including preliminary pilot runs and unsuccessful early variants, the full research process used approximately 3.6K A100 GPU-hours.

\begin{table}[t]
\centering
\small
\setlength{\tabcolsep}{4.5pt}
\caption{\textbf{Approximate compute resources.}
Runtime is measured per run unless otherwise stated.}
\label{tab:compute_resources}
\resizebox{\linewidth}{!}{
\begin{tabular}{lcccccc}
\toprule
Experiment
& Backbone / Model
& GPUs
& Peak Mem. / GPU
& Steps / Eval Size
& Wall Time
& GPU-hours \\
\midrule
Main \textsc{TPC} training
& ViT-L/14 dual encoder
& $8\times$A100-80GB
& 67GB
& 160K steps
& 31h
& 248 \\
Fast ablations
& ViT-B/16 dual encoder
& $4\times$A100-80GB
& 43GB
& 100K steps
& 11h
& 44 \\
Hyperparameter sensitivity
& ViT-B/16 dual encoder
& $4\times$A100-80GB
& 44GB
& 80K--100K steps
& 9--12h
& 36--48 \\
Backbone generalization
& ViT / ConvNeXt / RN variants
& $4$--$8\times$A100-80GB
& 38--72GB
& 80K--160K steps
& 10--39h
& 40--312 \\
Retrained CLIP-style baselines
& matched backbone
& $4$--$8\times$A100-80GB
& 40--66GB
& 100K--160K steps
& 12--30h
& 48--240 \\
AutoAttack robustness evaluation
& ViT-L/14
& $4\times$A100-80GB
& 36GB
& ImageNet val
& 8.5h/model
& 34/model \\
COCO/Flickr retrieval evaluation
& ViT-L/14
& $1\times$A100-80GB
& 24GB
& full test split
& 1.2h/model
& 1.2/model \\
LVLM projector fitting
& LLaVA-1.5-7B stack
& $8\times$A100-80GB
& 58GB
& 12K steps
& 5.5h
& 44 \\
POPE/VQA evaluation
& LLaVA-1.5-7B stack
& $4\times$A100-80GB
& 48GB
& benchmark eval
& 3--6h/model
& 12--24/model \\
Controlled-deletion span extraction
& Llama-3.1-8B-Instruct
& $1\times$A100-80GB
& 31GB
& cached captions
& 4.5h
& 4.5 \\
Controlled-deletion retrieval eval
& ViT-L/14
& $1\times$A100-80GB
& 24GB
& deletion grid
& 2.8h/model
& 2.8/model \\
\bottomrule
\end{tabular}}
\end{table}

The most expensive components are main ViT-L/14 training, backbone generalization, and robustness evaluation. In contrast, the proposed core filter itself adds only $O(d^2)$ parameters and negligible inference overhead relative to the dual encoder. Multi-view training increases text-side computation, but the image encoder is evaluated once per image, so the dominant vision-side cost remains comparable to CLIP-style training under the same batch size and optimization steps.

\section{Additional Experimental Results}
\label{app:results}

% =========================
% Caption-choice sensitivity (Main paper)
% =========================
\subsection{Consensus Core Stability: Does Caption-Choice Sensitivity Decrease?}
\label{sec:caption_stability}
We test the central ``text as partial constraint'' hypothesis:
if the model aligns vision to a shared consensus core, retrieval should be less sensitive to which caption view is used as the query.

For each image $x_i$ with captions $\{y_{i,k}\}_{k=1}^{K}$, we run text-to-image retrieval using each caption as the query and record the rank of the ground-truth image in the gallery, denoted $r_{i,k}$ (1 is best).
We then compute two caption-sensitivity metrics:
\begin{equation}
\mathrm{RankStd}(i)=\mathrm{Std}\big(\{r_{i,k}\}_{k=1}^K\big),
\qquad
\mathrm{WorstGap}(i)=\max_k r_{i,k}-\min_k r_{i,k}.
\end{equation}
Lower values indicate that the model is more stable to caption choice.
We report distributions over images and compare CLIP-style training vs.\ \textsc{TPC}.

\begin{figure}[t]
\centering
\includegraphics[width=0.96\linewidth]{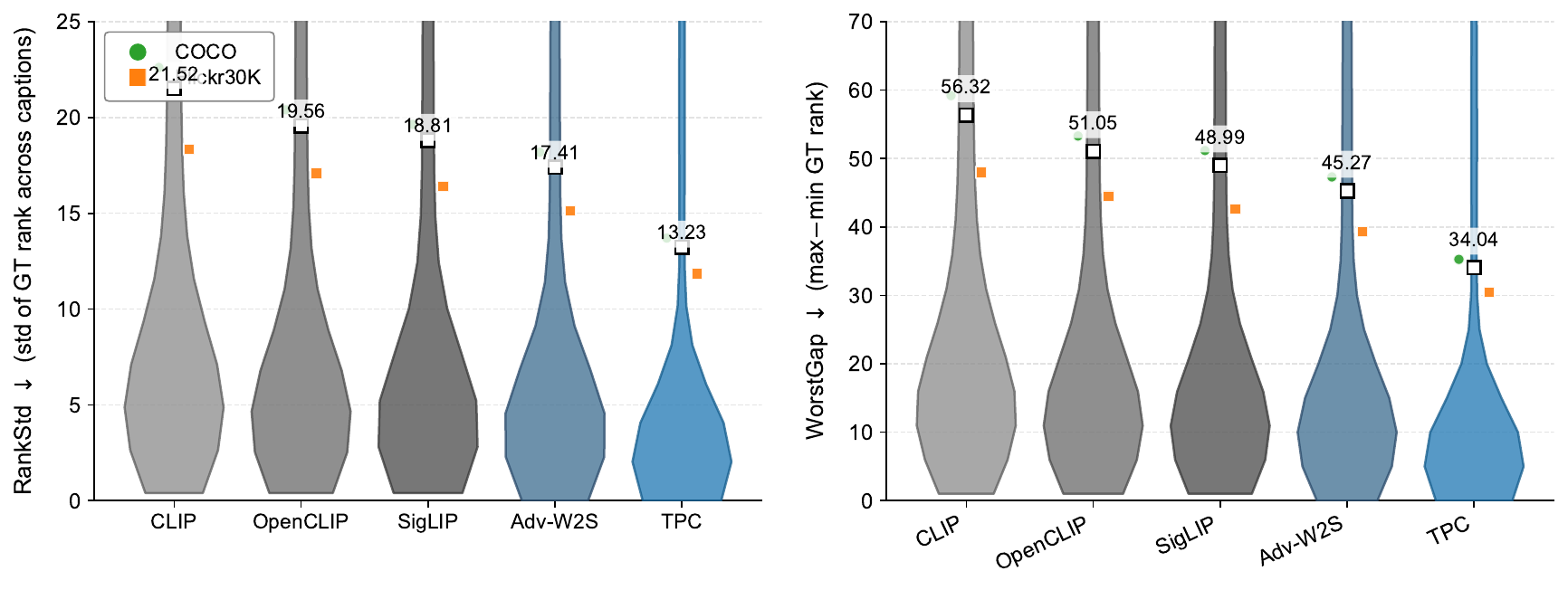}
\caption{\textbf{Caption-choice sensitivity.}
Violin plots show per-image distributions of RankStd and WorstGap on COCO and Flickr30K.
Dots and error bars indicate mean and 95\% CI; faint points show dataset-wise means.
Lower is better: less variation in retrieval rank across captions of the same image.}
\label{fig:caption_sensitivity}
\end{figure}

Fig.~\ref{fig:caption_sensitivity} shows that \textsc{TPC} substantially reduces caption-choice sensitivity:
both RankStd and WorstGap distributions shift downward and exhibit shorter tails than CLIP,
indicating fewer brittle failures driven by view-specific details.
This supports our design choice of aligning to a consensus core while suppressing residual commitment.

\subsection{Single-Caption Core Filter: Does $\psi_\theta$ Recover the Multi-View Consensus?}
\label{sec:core_filter_cdf}

We verify that the single-view core filter $\psi_\theta$ is not a shortcut, but actually learns to approximate the
multi-view intersection semantics captured by the consensus core $c$.

For each image $x_i$ with captions $\{y_{i,k}\}_{k=1}^{K}$, we compute the consensus core
$c_i=\mathrm{norm}(\sum_{k} t_{i,k})$ with $t_{i,k}=\mathrm{norm}(\phi_t(y_{i,k}))$.
For each caption view, we predict a single-caption core $\hat c_{i,k}=\mathrm{norm}(\psi_\theta(t_{i,k}))$ and measure
\begin{equation}
\rho_{i,k} \;=\; \cos(\hat c_{i,k},c_i) \;=\; \hat c_{i,k}^{\top}c_i \in [-1,1].
\end{equation}
We plot the empirical CDF $F(\alpha)=\mathbb{P}(\rho\le \alpha)$ across caption views.
As baselines, we include an \textbf{Identity} ``filter'' ($\hat c=t$) and an \textbf{w/o-agree} variant that removes the agreement loss ($\lambda_{\mathrm{agree}}{=}0$).
A better core filter yields a CDF that is shifted right (more mass at high $\rho$), consistently across datasets.

\begin{figure}[t]
\centering
\includegraphics[width=0.8\linewidth]{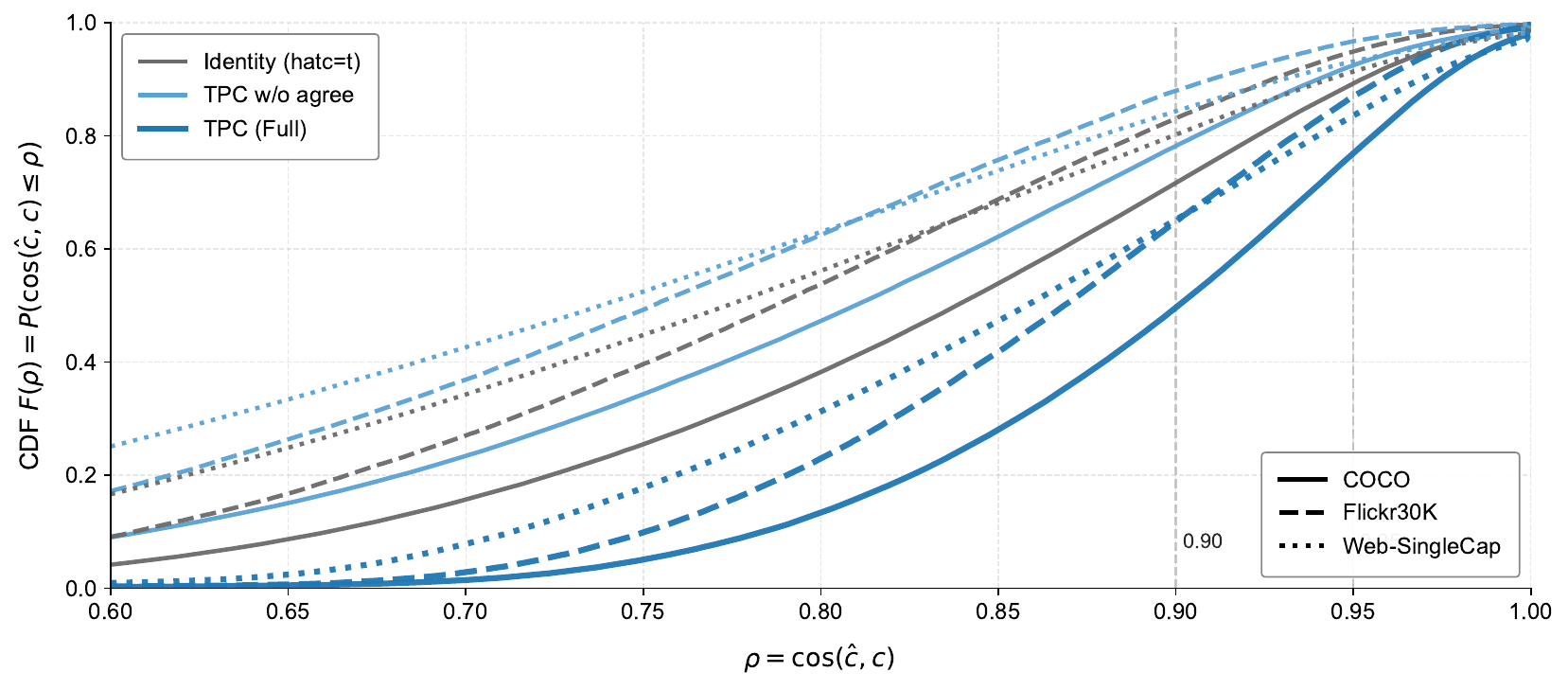}
\caption{\textbf{Does $\psi_\theta$ predict the consensus core?}
Empirical CDFs of $\rho=\cos(\hat c,c)$ on COCO, Flickr30K, and a single-caption web-style set.
Colors denote methods; line styles denote datasets.
A right-shifted curve indicates that single-caption cores $\hat c$ better match the multi-view consensus $c$.}
\label{fig:core_filter_cdf}
\end{figure}

Fig.~\ref{fig:core_filter_cdf} shows that \textsc{TPC} yields a clear right-shift in the $\cos(\hat c,c)$ distribution across datasets,
while the Identity baseline exhibits a heavier low-similarity tail.
Removing agreement supervision collapses much of the gain, indicating that $\psi_\theta$ learns a deployable single-caption estimator of shared semantics rather than exploiting dataset artifacts.

\subsection{When Helpful: Gain vs.\ Text-Only Solvability}
\label{sec:when_helpful_text_bias}

\emph{when is \textsc{TPC} most useful?}
If text is underspecified (weak constraint), a robust aligner should benefit more from suppressing residual commitment;
if text is already sufficient (strong bias), gains should diminish.

For each VQA query $q=(x,y)$, we compute a \emph{text-only solvability} score by running the same LVLM with the image masked
(or replaced by a null image) and taking the maximum answer confidence:
$s_{\text{text}}(q)=\max_a p(a\mid y,\varnothing)\in[0,1]$.
We then evaluate the full LVLM with (i) a baseline vision encoder and (ii) the \textsc{TPC} vision encoder, and compute
the \emph{accuracy gain} within bins of similar $s_{\text{text}}$.
Concretely, we sort examples by $s_{\text{text}}$ and form equal-count bins; each plotted point is a bin with
$x=\mathbb{E}[s_{\text{text}}]$ and $y=\Delta\mathrm{Acc}=\mathrm{Acc}_{\textsc{TPC}}-\mathrm{Acc}_{\text{Base}}$ (in percentage points).

\begin{figure}[t]
\centering
\includegraphics[width=0.85\linewidth]{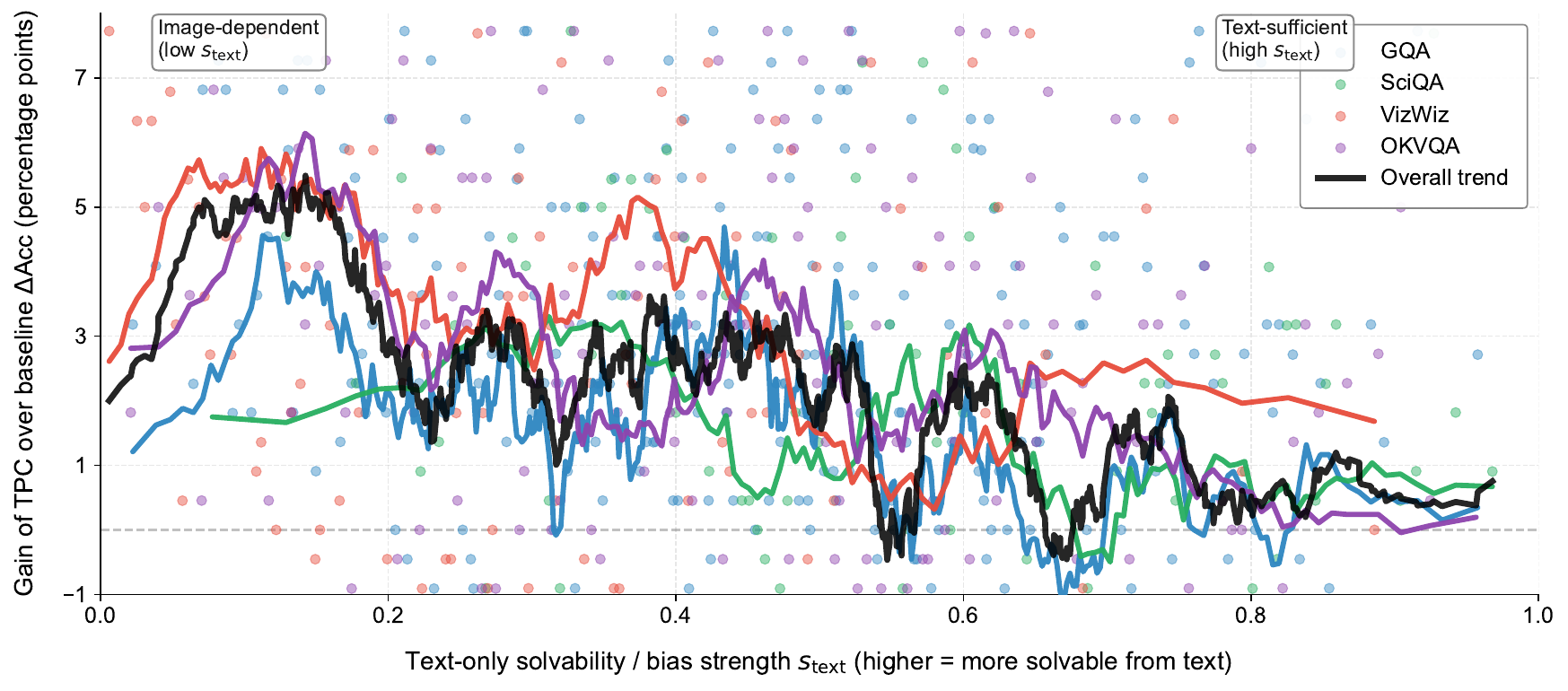}
\caption{\textbf{Gain vs.\ text-only solvability.}
Each point is an equal-count bin (hundreds of bins per dataset); colors indicate datasets.
The trend line is a smoothed running mean over bins (same color per dataset) and an overall trend (black).
Higher gain at lower solvability indicates \textsc{TPC} is most beneficial when text is a weak/underspecified constraint.}
\label{fig:when_helpful_textsolv}
\end{figure}

Fig.~\ref{fig:when_helpful_textsolv} shows a clear pattern: gains are largest when $s_{\text{text}}$ is low (image-dependent queries),
and shrink toward zero as $s_{\text{text}}$ increases (text already sufficient or strongly biased).
This supports our hypothesis that \textsc{TPC} is particularly helpful under underspecified language supervision,
where suppressing residual alignment prevents brittle over-commitment.

\subsection{Backbone/Scale Generalization}
\label{sec:backbone_generalization}
We test whether \textsc{TPC} is robust to architectural choices by swapping the vision backbone while keeping the
training recipe as consistent as possible.
We train \textsc{TPC} with identical objectives and a matched schedule across
ViT-S/16, ViT-B/16, ViT-L/14, ViT-H/14, ConvNeXt-B, ConvNeXt-L, and RN50.
We evaluate (i) ImageNet Top-1 clean and AutoAttack robust accuracy ($\epsilon{=}2/255$),
and (ii) MS-COCO text$\rightarrow$image Recall@\{1,5,10\}.
We visualize results with a parallel-coordinates plot where each polyline corresponds to one backbone
(faint lines: individual seeds; bold line: mean over 3 seeds).

\begin{figure}[t]
\centering
\includegraphics[width=0.95\linewidth]{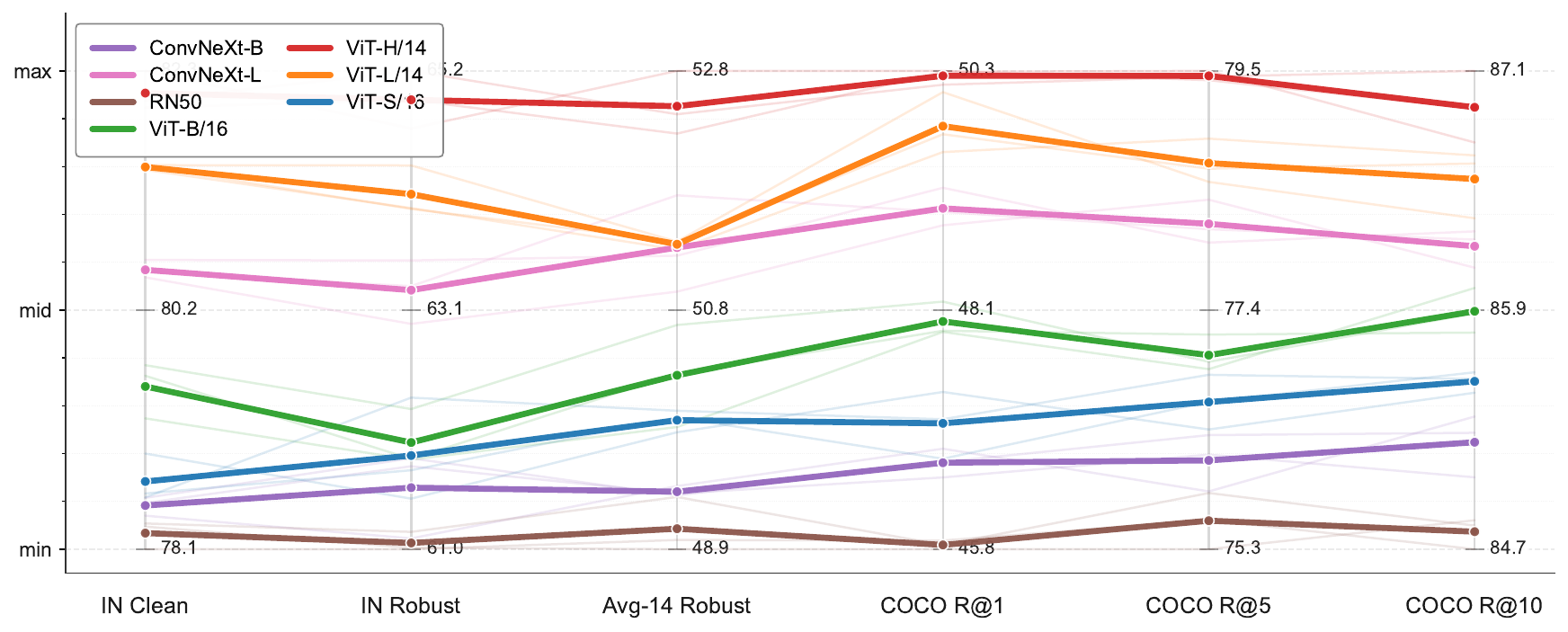}
\caption{\textbf{Cross-backbone robustness and retrieval.}
Parallel coordinates over multiple metrics (each axis is min--max normalized for readability).
Each backbone is shown with three seed runs (faint) and the mean (bold).
Consistent trends across architectures indicate \textsc{TPC} does not rely on a particular backbone family.}
\label{fig:backbone_parallelcoords}
\end{figure}

Fig.~\ref{fig:backbone_parallelcoords} shows that \textsc{TPC} yields stable performance across both ViT and ConvNeXt families:
larger backbones improve clean accuracy and COCO recall as expected, while robust accuracy remains consistently high without collapse.
This suggests the core--residual decomposition and uncertainty modulation transfer across architectural inductive biases.

\subsection{Systematic Missing-Detail Stress Test via Controlled Caption Deletion}
\label{sec:stress_deletion}

We perform a direct \emph{text-as-partial-constraint} stress test by systematically removing critical details from captions.
If a method over-commits to view-specific textual fragments, its performance and calibration should degrade sharply as deletions increase.
In contrast, \textsc{TPC} should degrade more gracefully due to residual suppression and uncertainty-aware smoothing.

Given a caption $y$, we apply a controlled deletion operator that removes a fixed fraction of tokens belonging to one of three semantic categories:
\textbf{Attributes} (color/size/material/adjectives), \textbf{Relations} (spatial/interaction predicates), and \textbf{Counts} (numerals/quantifiers).
We test deletion ratios in $\{10\%,30\%,50\%,70\%\}$ and evaluate text$\rightarrow$image retrieval on COCO/Flickr30K.
For each condition, we report \textbf{Recall@1} (performance; higher is better) and \textbf{ECE} (calibration; lower is better)
computed from retrieval confidence (top-1 softmax probability over a standardized candidate set).

\textbf{Data construction: LLM-assisted controlled deletion.}
The controlled deletion operator is built with a text-only span extractor driven by an off-the-shelf instruction-tuned LLM.
Concretely, we use \textbf{Llama-3.1-8B-Instruct} to tag verbatim spans in the caption that correspond to
\textit{Attributes}, \textit{Relations}, and \textit{Counts}, and we then perform \emph{programmatic} deletions to meet the target ratio.
We run the extractor with deterministic decoding (temperature $0$, top-$p{=}1$) and cache all extracted spans for reproducibility.
To enforce \emph{strict controllability}, the LLM is not allowed to paraphrase or introduce new content: it only returns spans that
appear exactly in the original caption.
Given the returned spans for a target type, we uniformly sample spans (with a fixed random seed) until the removed token count
reaches the specified deletion ratio (10/30/50/70\%) of all tokens covered by that type; if the caption contains too few tokens of a type,
we delete all available ones.
Finally, we apply a minimal deterministic cleanup (collapse whitespace, remove dangling commas/prepositions) without any rewriting.

\begin{tcolorbox}[
  colback=black!2,
  colframe=black!20,
  boxrule=0.6pt,
  arc=2mm,
  left=1.2mm,right=1.2mm,top=0.8mm,bottom=0.8mm,
  title=\textbf{Prompt template for LLM span extraction (used to build deletion sets)},
  fonttitle=\small\bfseries,
  fontupper=\ttfamily\small
]
System: You are a precise annotation tool. You must not rewrite text.
You only copy exact substrings that already appear in the caption.

User:
Given the following image caption, extract verbatim spans that belong to three semantic categories:

(1) ATTRIBUTES: colors, sizes, materials, adjectival modifiers, descriptive properties (e.g., "red", "wooden", "small", "striped").
(2) RELATIONS: spatial or interaction relations/predicates connecting entities (e.g., "next to", "on top of", "holding", "behind").
(3) COUNTS: numerals and quantifiers (e.g., "two", "3", "several", "a pair of").

Rules:
- Output MUST be a single valid JSON object with keys: "attributes", "relations", "counts".
- Each value is a list of strings, and each string MUST be an exact substring copied from the caption (verbatim, case-preserving).
- Spans should be minimal and non-overlapping. Do NOT invent spans that are not present.
- If a category has no span, output an empty list for that key.
- Do NOT output anything other than the JSON.

Caption:
\verb|<<<CAPTION_TEXT>>>|
\end{tcolorbox}

\begin{figure}[t]
\centering
\includegraphics[width=0.98\linewidth]{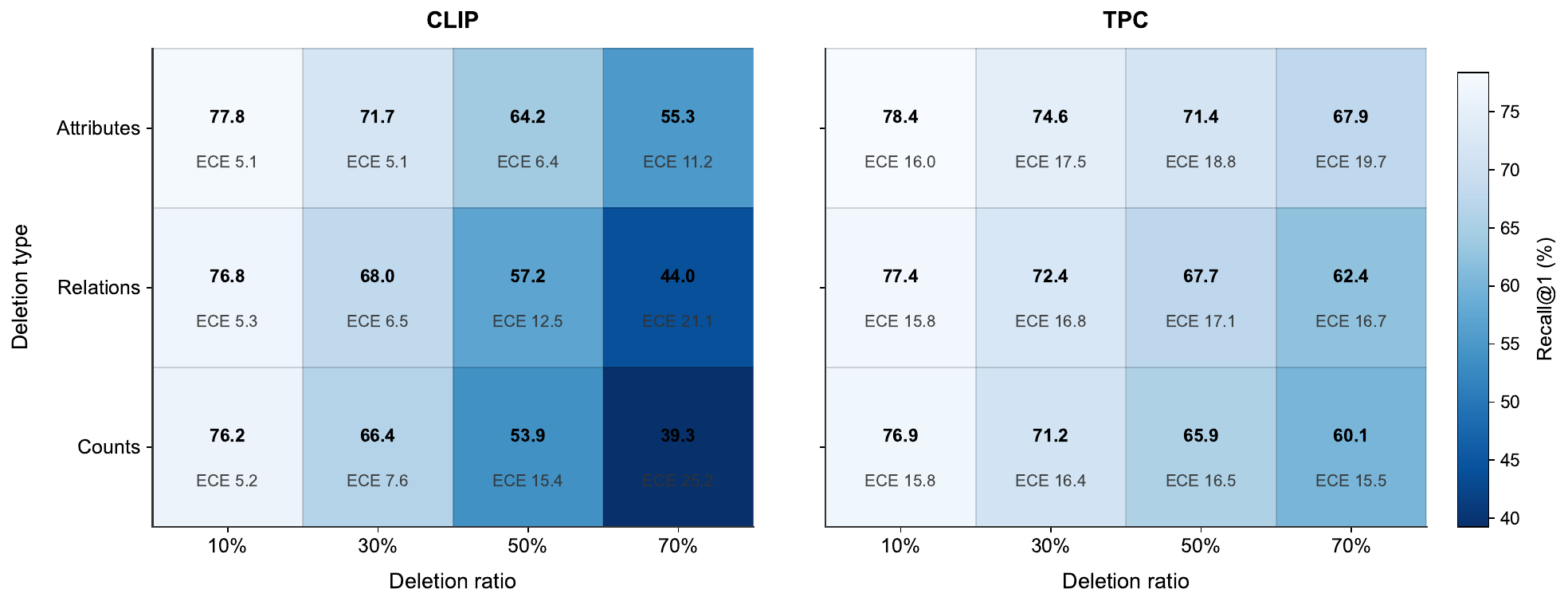}
\caption{\textbf{Controlled caption deletion stress test.}
Heatmaps show Recall@1 (\%) under increasing deletion ratios (x-axis) and deletion types (y-axis).
Each cell is annotated with \emph{Recall@1 / ECE} (ECE shown as a small number).
\textsc{TPC} maintains higher accuracy and lower ECE as deletions grow, indicating robustness to systematically weakened language constraints.}
\label{fig:detail_deletion_heatmap}
\end{figure}

Fig.~\ref{fig:detail_deletion_heatmap} shows that \textsc{TPC} is consistently more resilient to missing details:
(1) performance drops are milder under severe deletions, especially for \emph{relations} and \emph{counts} where underspecification is most damaging;
(2) calibration degrades more slowly, suggesting $\hat u$-modulated sharpness prevents overconfident retrieval when the text constraint is weakened.
Overall, the stress test supports our thesis that language supervision is partial and that explicitly suppressing residual commitment improves robustness.

%%%%%%%%%%%%%%%%%%%%%%%%%%%%%%%%%%%%%%%%%%%%%%%%%%%%%%%%%%%%
\section{Responsible Research and Reproducibility}
\label{app:responsible}
\subsection{Scope and Limitations}
\label{app:scope_limitations}

\textsc{TPC} is designed as a representation-level alignment method for settings where language supervision is naturally partial or can be instantiated through multiple faithful caption views. Its main scope is therefore robust image--text matching, calibrated retrieval, zero-shot transfer, and plug-in transfer to LVLMs through a stronger vision encoder. The method is not intended to replace task-specific safety filters, factuality checkers, or policy-level controls in deployed LVLM systems; rather, it provides a more reliable alignment backbone that can complement these downstream safeguards.

The method uses multi-view captions during training to estimate a consensus semantic core. When human multi-caption annotations are available, we use the provided views directly; when only one caption is available, we construct faithful auxiliary views through controlled deletion, paraphrasing, and back-translation, followed by deterministic filtering. This setting matches many existing vision--language resources and can be extended with retrieval- or generation-based view construction. In our experiments, performance is stable across $K\in\{3,4,5\}$ and saturates around four to five views, suggesting that \textsc{TPC} does not require a large number of captions per image.

Computationally, \textsc{TPC} adds text-side view encoding and a lightweight linear core filter, while sharing the image forward pass across views. As a result, the additional overhead is moderate compared with full adversarial training or LVLM instruction tuning. In downstream use, inference remains single-text and uses only the predicted core embedding, so the deployment cost is essentially the same order as a standard CLIP-style dual encoder.

\subsection{Statistical Reporting Across Seeds}
\label{app:statistical_reporting}

Unless otherwise stated, trainable \textsc{TPC} variants are run with three random seeds. The seeds affect model initialization of the core filter, data shuffling, caption-view subsampling, and minibatch ordering. We report mean values in the main result tables for readability and use mean$\pm$standard deviation in sensitivity and diagnostic figures. For metrics whose distributions are asymmetric, such as residual leakage and rank-based quantities, we additionally use bootstrap confidence intervals over examples or caption views.

Table~\ref{tab:seed_variability_tpc} reports seed-level variability for the principal \textsc{TPC} metrics. The standard deviations are small relative to the margins over the strongest baselines, indicating that the main conclusions are not driven by a single favorable run.

\begin{table}[t]
\centering
\small
\setlength{\tabcolsep}{5.5pt}
\caption{\textbf{Seed variability for main \textsc{TPC} results.}
Mean$\pm$standard deviation over three random seeds.}
\label{tab:seed_variability_tpc}
\resizebox{\linewidth}{!}{
\begin{tabular}{lcccccc}
\toprule
Setting
& ImageNet Clean
& ImageNet Robust
& Avg-14 Clean
& Avg-14 Robust
& POPE Avg F1
& VQA Avg \\
\midrule
\textsc{TPC} (ViT-L/14)
& $81.42{\pm}0.12$
& $64.05{\pm}0.27$
& $76.19{\pm}0.18$
& $52.03{\pm}0.31$
& $85.16{\pm}0.22$
& $61.25{\pm}0.19$ \\
\bottomrule
\end{tabular}}
\end{table}

For external checkpoints that are evaluated without retraining, we report deterministic evaluation numbers under the same preprocessing and evaluation scripts. For retrained baselines and ablations, we use the same three-seed protocol whenever the corresponding method is computationally comparable.

\subsection{Broader Impact and Responsible Use}
\label{app:broader_impacts}

\textsc{TPC} aims to improve the reliability of vision--language representations under underspecified language supervision. Potential positive impacts include more stable image--text retrieval, better calibrated confidence under ambiguous captions, and reduced object hallucination when the resulting vision encoder is used inside LVLM systems. These properties are useful for applications where users provide incomplete or paraphrased descriptions and where overconfident visual grounding can lead to unreliable downstream responses.

The same improvements may also increase the capability of downstream multimodal systems. If such systems are deployed without proper safeguards, stronger visual grounding could still be misused in settings such as surveillance, automated profiling, or misleading content generation. Moreover, calibrated confidence should not be interpreted as a guarantee of factual correctness or safety. We therefore view \textsc{TPC} as a representation-level reliability component rather than a complete deployment solution.

Responsible use should combine \textsc{TPC} with application-level safety filters, dataset and demographic audits, human oversight in high-stakes settings, and clear communication of model uncertainty. We do not release a new user-facing model or scraped dataset at submission time.

\subsection{LLM Usage Disclosure}
\label{app:llm_usage}

The core \textsc{TPC} method does not use an LLM to generate training labels, optimize model parameters, or produce predictions. The only methodological use of an LLM is in the controlled caption-deletion stress test in Appendix~\ref{sec:stress_deletion}. Specifically, we use Llama-3.1-8B-Instruct as a deterministic span extractor to identify verbatim caption spans corresponding to attributes, relations, and counts. The LLM is constrained to return exact substrings from the original caption and is not allowed to paraphrase, introduce new content, or modify the image--text label.

\end{document}